\documentclass[letterpaper]{article} 

\usepackage{aaai23_arxiv}

\usepackage{times}  
\usepackage{helvet}  
\usepackage{courier}  
\usepackage[hyphens]{url}  
\usepackage{graphicx} 
\usepackage{natbib}  
\usepackage{caption} 
\usepackage[noend]
{algpseudocode}
\usepackage{multirow}
\usepackage{algorithm}

\usepackage{newfloat}
\usepackage{listings}
\DeclareCaptionStyle{ruled}{labelfont=normalfont,labelsep=colon,strut=off} 
\floatstyle{ruled}
\newfloat{listing}{tb}{lst}{}
\floatname{listing}{Listing}
\title{Local Intrinsic Dimensional Entropy}
\author{
    Rohan Ghosh\textsuperscript{\rm 1}, 
    Mehul Motani\textsuperscript{\rm 1,2}
}
\affiliations{
\textsuperscript{\rm 1}    
College of Design and Engineering, Department of Electrical and Computer Engineering, National University of Singapore\\
\textsuperscript{\rm 2}    
N.1 Institute for Health, Institute for Digital Medicine (WisDM), Institute of Data Science, National University of Singapore

    rghosh92@gmail.com, motani@nus.edu.sg
}

\usepackage{graphicx}
\usepackage{amsmath}
\usepackage{cancel}
\usepackage{enumitem}
\usepackage{subcaption}
\usepackage{caption}
\usepackage{graphicx}
\usepackage{adjustbox}
\usepackage{tablefootnote}
\usepackage{enumitem}
\usepackage{amsthm}
\usepackage{enumitem}
\usepackage{wrapfig,lipsum,booktabs}
\usepackage{amssymb}

\newtheorem{theorem}{Theorem}
\newtheorem{corollary}{Corollary}[theorem]

\newtheorem{prop}{Proposition}
\newtheorem{remark}{Remark}

\newtheorem{definition}{Definition}

\DeclareMathOperator{\R}{\mathbb{R}}
\usepackage{mathtools}
\DeclareMathOperator*{\E}{\mathbb{E}}

\usepackage{tikz}
\usetikzlibrary{calc}
\newcommand{\hcancel}[5]{%
    \tikz[baseline=(tocancel.base)]{
        \node[inner sep=0pt,outer sep=0pt] (tocancel) {#1};
        \draw[black] ($(tocancel.south west)+(#2,#3)$) -- ($(tocancel.north east)+(#4,#5)$);
    }%
}%

\newcommand{\Hbar}{\hcancel{\textit{h}}{-1pt}{4pt}{-1pt}{-0.5pt}}
\newcommand{\Ibar}{\hcancel{\textit{I}}{-1pt}{3pt}{2pt}{-0.5pt}}

\DeclareMathOperator{\arrow}{\xrightarrow[]{}}

\usepackage{mathtools} 
\usepackage{extarrows}

\usepackage{latexsym}
\usepackage{centernot}
\def\fr{\ensuremath{\leadsto}}
\def\notfr{\centernot\fr}
\def\frf{\xleftrightarrow{}}

\begin{document}

\maketitle

\begin{abstract}
Most entropy measures depend on the spread of the probability distribution over the sample space $\mathcal{X}$, and the maximum entropy achievable scales proportionately with the sample space cardinality $|\mathcal{X}|$. For a finite $|\mathcal{X}|$, this yields robust entropy measures which satisfy many important properties, such as invariance to bijections, while the same is not true for continuous spaces (where $|\mathcal{X}|=\infty$).
Furthermore, since $\R$ and $\R^d$ ($d\in \mathbb{Z}^+$) have the same cardinality (from Cantor's correspondence argument), cardinality-dependent entropy measures cannot encode the data dimensionality.
In this work, we question the role of cardinality and distribution spread in defining entropy measures for continuous spaces, which can undergo multiple rounds of transformations and distortions, e.g., in neural networks. We find that the average value of the local intrinsic dimension of a distribution, denoted as ID-Entropy, can serve as a robust entropy measure for continuous spaces, while capturing the data dimensionality. We find that ID-Entropy satisfies many desirable properties and can be extended to conditional entropy, joint entropy and mutual-information variants. ID-Entropy also yields new information bottleneck principles and also links to causality. In the context of deep learning, for feedforward architectures, we show, theoretically and empirically, that the ID-Entropy of a hidden layer directly controls the generalization gap for both classifiers and auto-encoders, when the target function is Lipschitz continuous. Our work primarily shows that, for continuous spaces, taking a structural rather than a statistical approach yields entropy measures which preserve intrinsic data dimensionality, while being relevant for studying various architectures.
\end{abstract}

\section{Introduction and Motivation} \label{sec:intro}
In this paper, we consider an alternative interpretation of what outlines a good measure of information, particularly in the context of sample spaces which can repeatedly undergo continuous transformations of any kind. This applies to domains where data can undergo multiple rounds of processing and distortion, for example, in machine learning with deep neural networks. Our objective here is to find whether a robust information measure can be defined in this context, which ideally satisfies the various properties exhibited by discrete entropy, which is defined for discrete spaces.  

Entropy for discrete sample spaces $\mathcal{X}$ ($|\mathcal{X}|<\infty$) is essentially a fundamental property of the space. If one were to map the elements of a finite $\mathcal{X}$ to another space $\mathcal{X}'$ using a bijection, $\mathcal{X}'$ will have the same entropy as $\mathcal{X}$. Therefore, bijective transformations such as shift and scaling will yield the same discrete entropy. Furthermore, the notion of \textit{surprise}, which is often associated with entropy, aptly applies to the case of finite $\mathcal{X}$, as a larger spread of the distribution essentially points to less certainty about the outcome of any random draw from $\mathcal{X}$. It is also worth noting that any deterministic function, which maps a finite $\mathcal{X}$ to another finite space $\mathcal{X}'$, can only reduce discrete entropy, as deterministic maps should not increase the amount of information. 

It is interesting to note then, that the counterpart of entropy measures for continuous spaces, such as differential entropy and other variants \cite{60ren,shannon_orig}, don't satisfy the same desirable properties as discrete entropy. First, we note that differential entropy can change significantly in response to trivial bijective maps, such as simply scaling the axes by a factor of $\alpha$. In fact, for a continuous random variable (RV) $X\in \R^d$, we have that $h(\alpha X)=h(X) + \log |\alpha|$, where $h(.)$ denotes the differential entropy. Second, we reason that the notion of distribution spread cannot be used as a measure of surprise for continuous distributions. This is because if one were to simply scale the axes by $\alpha$, we can directly tune the spread of the distribution. By doing so, one would be able to yield arbitrarily higher measures of surprise by simply processing the RV $X$ with a bijective map such as scaling. Note that we cannot do this for discrete entropy, as bijections such as scaling and translation do not alter the discrete entropy of the space. Third, similar to our previous arguments, we see that simply scaling the variable to have a much higher spread can increase differential entropy arbitrarily, and thus processing can increase differential entropy, contrary to how discrete entropy behaves. Note that similar conclusions would follow even if one were to estimate the differential entropy after normalizing the variance of $X$, by replacing the scaling transformation with any other scale-preserving distortion. 

\noindent\textbf{Towards Local Measures of Surprise:} We next consider how to quantify surprise for continuous spaces. We start by noting that maximum entropy for a discrete set $\mathcal{X}$ is $\log |\mathcal{X}|$, and thus depends on the cardinality of the set. In contrast, for continuous spaces of $d$ dimensions $\R^d$, the cardinality is actually invariant to the dimension $d$. As observed first by Cantor in his 1874 work \cite{cantor_1874}, there exists a bijective transformation between $\R^{d_1}$ and $\R^{d_2}$ for any $d_1,d_2\in \mathbb{Z}^+$, indicating that $\R^{d_1}$ and $\R^{d_2}$ have the same cardinality.  These observations point to a potentially different interpretation of what an entropy measure could signify in the context of continuous spaces. We see that both cardinality and spread cannot be used as a basis for computing any measure of \textit{surprise}. Furthermore, we observe that the notion of \textit{surprise} itself cannot be immediately defined in this case, as the true surprise $\lim_{dx\xrightarrow[]{} 0}log\frac{1}{P(x)dx}$, for any element $x\in \mathcal{X}$ would be infinite, as $dx\xrightarrow[]{} 0$ implies $P(x)dx\xrightarrow[]{} 0$ as well (except for Dirac-delta distributions). 

This leads us to question whether there are alternate ways to quantify surprise for continuous spaces. To do so, we define the \textit{covering} of a distribution $P(X)$, as the minimal number of spheres $B_0,B_1,...,B_K$ of fixed radius $\epsilon$, which together cover the entire support of $P(X)$. By doing so, we can effectively treat $P(X)$ as a discrete distribution of cardinality $K$, and we can define surprise in the conventional \textit{log-probability} sense. Next, for any point $X_0$ which lies in the support of $P(X)$, we define a \textit{local-neighborhood} $\mathcal{N}(X_0)$ centered at $X_0$ , which is a sphere of radius $\delta$, s.t. $\delta=\tau\epsilon$ for some constant $\tau>1$. Let the number of spheres among  $B_0,B_1,...,B_K$ that lie inside $\mathcal{N}(X_0)$, be denoted by $n$. As $\epsilon \arrow 0$, we can easily show that $n\arrow (\delta/\epsilon)^{d(X_0)}$, where $d(X_0)$ is the dimension of the distribution manifold around $X_0$, also called the local intrinsic dimension \cite{local_id_survey}. As such, the \textit{local surprise} converges to $\log n = \log (\delta/\epsilon)^{d(X_0)} = d(X_0) \log (\delta/\epsilon)=d(X_0) \log (\tau)$. Furthermore, the expected value of local surprise is $\mathbb{E}_{X\sim P}\left[d(X)\right] \log (\tau)$. 
The quantity $\mathbb{E}_{X\sim P}\left[d(X)\right]$ is what we propose to analyze in this work. We term this quantity as the \textit{local intrinsic dimensional entropy} (ID-Entropy) of the distribution, and find that it satisfies all of the earlier discussed properties of discrete entropy. Furthermore, we also find that ID-entropy can be relevant in the context of analyzing the information in deep neural network layers, and predicting their generalization behaviour, which is discussed next.

\noindent\textbf{Information Bottlenecks and Generalization:} The \textit{Information Bottleneck} (IB) principle, originally proposed in \cite{tishby_bottle}, discusses the relevance of quantities such as the mutual information (MI) between the input $X$ and features $T$, denoted by $I(X;T)$, and between the features and labels $Y$, denoted by $I(T;Y)$ to generalization performance. Broadly, the IB principle hypothesizes that neural networks undergo two phases during learning: (a) a \textit{fitting} phase where both $I(X;T)$ and $I(Y;T)$ increases, and (b) a \textit{compression} phase where $I(X;T)$ decreases. However, the true MI $I(X;T)$ is actually infinity, as given a fixed network configuration, $T$ is a deterministic function of $X$, and $I(X;f(X))=\infty$ for deterministic $f$ \cite{mi_in}. Furthermore, these two phases have not always been observed in networks \cite{better_ib}. In this work, we show that ID-Entropy of the feature layer $T$ of a neural network, denoted by $\Hbar(T)$, can be analyzed similarly, leading to analogous bottleneck criteria, and new generalization error bounds. Furthermore, unlike the vacuous nature of the estimate of $I(X;T)$, $\Hbar(T)$ is non-zero and finite, and has \textit{opposite trends} during training compared to its IB counterpart. From ID-Entropy's perspective, we find that the network undergoes compression first, as $\Hbar(T)$ reduces initially and increases slowly thereafter. Furthermore, both in theory and practice, we find that if the network converges to a higher $\Hbar(T)$ after training, it shows worse generalization performance.




\section{Contributions}
The following are the specific contributions of our work. 
\begin{enumerate}[nosep,leftmargin=*,wide=0pt]
    \item \textbf{Definition and Properties:} A novel measure of entropy for continuous domains called \textit{ID-Entropy} is proposed. Unlike differential entropy and other counterparts, \textit{ID-Entropy} is motivated from the observation that the intrinsic dimensionality of the data relates to a local measure of information gain, yielding a new notion of surprise for continuous spaces. We find that \textit{ID-Entropy} satisfies many properties of discrete Shannon entropy itself.
    \item \textbf{New Information Bottleneck Criteria:} \textit{ID-Entropy} measures leads to a new notion of information bottlenecks for auto-encoder and classifier architectures. Unlike conventional bottlenecks where the notion of MI and differential entropy can be vacuous, \textit{ID-Entropy} avoids this problem, while yielding new perspectives on the problem.
    \item \textbf{Connection to Generalization:} We find that the \textit{ID-Entropy} of input and hidden layers for feedforward architectures in large dataset scenarios controls the generalization gap for both auto-encoders and classifier architectures.
    \item \textbf{Connection to Causality:} \textit{ID-Entropy} encodes the underlying causal factors responsible for generating any given dataset, when considering a class of causal models where causal effects are generated by continuous functions. This pertains to the cases where, if $X$ is the only cause of $Y$, $Y$ is a continuous function of $X$ (but not necessarily vice-versa). 
    \item \textbf{Experimental Validation:} We empirically compute the \textit{ID-Entropy} of feature layers for feed-forward architectures such as Convolutional Neural Networks (CNNs) and Convolutional Auto-Encoders trained on MNIST and CIFAR-10. For both auto-encoders and classifiers, we find that the estimated ID-Entropy correlates with the generalization gap.  
\end{enumerate}
\section{\textit{ID-Entropy}: A New Perspective on Entropy}
\subsection{Assumptions and Definitions} \label{sec:definitions}
In this section, we present some useful definitions and assumptions, building up to \textit{ID-Entropy}.

Let $X$ be a random variable (RV) and $P(X)$ be the distribution of $X$. We first enumerate some useful notions.
\textbf{(A1)} For a RV $X$, the support of $P(X)$ can be expressed as a union of a finite number of \textit{locally Euclidean} \cite{lc_eu} manifolds, where locally Euclidean means that every point in the support of $P(X)$ has a local neighborhood which is homeomorphic to $\R^k$ for some non-negative integer $k$. 
\textbf{(A2)} For a set of RVs $\mathbb{X}=\{X_1,X_2,..,X_n\}$ and for any $S\subseteq\mathbb{X}$, the geodesic distance between any two connected points in $P(S)$ is finite. 
\textbf{(A3)} For a RV $X$, the number of connected components in the support of $P(X)$ is countable. 
Unless otherwise mentioned, all RVs in this work are assumed to satisfy A1, A2, and A3.
These assumptions are not very stringent, and we expect them to be satisfied for most cases. We provide a more detailed discussion regarding the same in the Appendix, available in \cite{arxiv_id_entropy}.  




Next, we define a novel relation between two RVs called \textit{f-relatedness}. Note that f-relatedness is used in the context of vector fields with a different meaning \cite{vec_frel}.

\begin{definition}
(\textbf{f-relatedness}) We say that a RV $Y$ is f-related to another RV $X$, if there exists a continuous function $f$ such that $f(X)=Y$.
This is denoted by $X\fr Y$.
We also refer to this as an f-relation. When $Y$ is not f-related to $X$, we denote this as $X\notfr Y$.
\end{definition}
Note that two RVs $X$ and $Y$ are either f-related or they are not. Furthermore, f-relatedness is asymmetric, i.e., $X\fr Y$ does not imply $Y\fr X$. Next, we define \textit{symmetric f-relations}.

\begin{definition}
(\textbf{Symmetric f-relation}) We say that RV's $X$ and $Y$ have a symmetric f-relation when  $X\fr Y$ and $Y\fr X$. A symmetric f-relation is denoted as $X \frf Y$. 
\end{definition}
Note that both f-relations and symmetric f-relations can be extended to the multiple RVs scenario. One can similarly define $(X_1,X_2,..,X_k)\fr Y$ and  $X\fr (Y_1,Y_2,..,Y_k)$ for RVs $X_1,X_2,..,X_k$ and $Y_1,Y_2,..Y_k$. We note that $(X_1,X_2,..,X_k)\fr Y$ does not imply $X_i\fr Y$ for any $i$. 

Next, we define the \textit{$\epsilon$-Neighborhood Intrinsic Dimension} of a RV $X$, at any point $\rho\in \R^d$, as follows. First, for some $0\leq\epsilon\leq\infty$, let us consider another RV $X_{\epsilon}^{\rho}$, which follows a distribution $P_{\epsilon}^{\rho}$ given by:
\[
    P_{\epsilon}^{\rho}(X)= 
\begin{cases}
    \frac{P(X)}{c},& \text{if } \lVert X-\rho \rVert \leq \epsilon\\
    0,              & \text{otherwise}
\end{cases}
\]
Here, $c$ is a constant which scales the distribution to ensure $\int P_{\epsilon}^{\rho}(X) dX=1$. 


Given this, we define the \textit{$\epsilon$-Neighborhood Intrinsic Dimension}  $d_{\epsilon}(\rho)$ as follows:
\begin{equation} \label{eq:master}
    d_{\epsilon}(\rho) = \min n \ \mbox{ s.t. } \ (v_1,v_2,..v_n) \frf X_{\epsilon}^{\rho}
\end{equation}
where $v_1,v_2,..v_N$ represent either continuous RVs such that $ 0 \leq v_i \leq 1$, or binary RVs such that $v_i\in \{0,1\}$. 
Thus, in other words, $d_{\epsilon}(\rho)$ represents the minimum number of total continuous and/or discrete RVs of any arbitrary dense distribution $Q$, which can be associated with $X_{\epsilon}^{\rho}$ via a homeomorphism, i.e., infinitesimal changes in $v_1,v_2,..v_N$ always yields infinitesimal changes in $X_{\epsilon}^{\rho}$ and vice versa. Note that $d_{\epsilon}(\rho)$ will be finite for small $\epsilon$ due to the locally Euclidean assumption. Thus, note that cycling through all possible values of the RVs and remapping them to $\R^d$ using the appropriate continuous map should precisely yield the set of all possible values of the RV $X_{\epsilon}^{\rho}$. 



With this, we can define \textit{$\epsilon$-ID-Entropy} $ \Hbar^{\epsilon}(X)$ of a RV $X$ with an underlying distribution $P(X)$ as follows:
\vspace{-1mm}
\begin{definition}
(\textbf{$\epsilon$-\textit{ID-Entropy}}). We define the {$\epsilon$-ID-Entropy} $\Hbar^{\epsilon}(X)$ of a RV $X$ as $\Hbar^{\epsilon}(X) = \E_{\rho\sim P(X)}\left [d_{\epsilon}(\rho)\right]$, where $P(X)$ is the underlying distribution of $X$.
\end{definition}

Note that the above formulation represents the \textit{ID-Entropy} for a finite fixed $\epsilon>0$, and, as such, one can similarly consider all possible values of $0< \epsilon < \infty$. Next, with this, we can define the \textit{ID-Entropy} $ \Hbar(X)$ of $X$ as follows.
\begin{definition}
(\textbf{ID-Entropy}). We define the \textit{ID-Entropy} $\Hbar(X)$ of a RV $X$ as
\begin{equation}
    \Hbar(X) =  \lim_{\epsilon\arrow 0^{+}}\Hbar^{\epsilon}(X).
\end{equation}
\end{definition}

Note that the above definition can be extended to the case of multiple RVs, just by considering their concatenated form as the new RV. Thus, the joint ID-Entropy of RV $X\in \R^d$ and $Y\in \R^l$ can be defined as follows:
\begin{equation}
    \Hbar(X,Y) = \Hbar(V),
\end{equation}
where $V\in \R^{d+l}$ is simply a concatenation of the two RVs $X$ and $Y$. 
Note that one can similarly define {\em mutual ID-information} and {\em conditional ID-Entropy} from the definition of \textit{ID-Entropy}, in the same way as their traditional counterparts.  They are defined as follows. 
\begin{definition}
    (\textbf{Mutual ID-information}) Given RVs $X$ and $Y$, 
$\Ibar(X;Y)=\Hbar(X)+ \Hbar(Y) - \Hbar(X,Y)$. Note that $\Ibar(X;Y)\geq0$. 
\end{definition}
\begin{definition} 
(\textbf{Conditional \textit{ID-Entropy}}) Given RVs $X$ and $Y$, 
$\Hbar(Y|X)=\Hbar(X,Y)- \Hbar(X)$. Note that $\Hbar(Y|X)\geq0$, and $\Hbar(Y|X)=0$ iff $X\fr Y$.
\end{definition}

Algorithm \ref{alg:one} shows how to estimate ID-Entropy, making use of a global-ID estimator, denoted by $f_{ID}$. 


\begin{algorithm}[t]
	\caption{Estimation of ID-Entropy}
	\textbf{Input:} $S=\{X_1,..,X_m\}$ (i.i.d samples of RV $X$), a global-ID estimator $f_{ID}(S')$ of points in $S'$, $\&$ parameters $(k,n)$. \newline
    \textbf{Output:} $ID_X$ (Estimate of ID-Entropy of $X$)
	\begin{algorithmic}[1]
	    \State $ID_{sum} = 0$ \;
        \For {$j=1,\lfloor 2m/n \rfloor,\lfloor 3m/n \rfloor,..\ldots,m$}
			\State Let $S'=$ $k$-nearest neighbors of $X_j$ in $S$\;
			\State $ID_{sum} = ID_{sum} + f_{ID}(S')$\;
		\EndFor
		\State $ID_X = ID_{sum}/n$\;
	\end{algorithmic}
	\label{alg:one}
\end{algorithm}

\subsection{Properties of \textit{ID-Entropy}}

We outline the properties of the proposed \textit{ID-Entropy} measure and its extensions. The proofs of all the properties, and all subsequent theoretical results in this paper, are given in the Appendix in \cite{arxiv_id_entropy}. In what follows, $dim(X)$ represents the dimensionality of $X$'s co-domain.

\begin{enumerate}[leftmargin=0.5cm,topsep=0cm,label={\bf P\arabic*},align=left]
\item $0\leq \Hbar(X)\leq dim(X)$ (Finite and Bounded)
\item $\Hbar(X)\leq \Hbar^{\epsilon}(X)$ for any $\epsilon>0$. 
\item $\Hbar(\alpha X)=\Hbar(X)$, for any real $\alpha$ (Scale-invariance). \label{FEP1}
\item $\Hbar(X,Y)=\Hbar(Y,X)$ (Symmetric). \label{FEP2} 
\item $\Hbar(X,Y)\geq \Hbar(X)$ and $\Hbar(X,Y)\geq \Hbar(Y)$.  
\item $\Hbar(X,Y)\leq \Hbar(X) +  \Hbar(Y)$.
\item If $X\fr Y$,  $\Hbar( X,Y) = \Hbar(X)\geq \Hbar(Y)$. (Cannot increase with continuous processing)
\item  If $X \fr Y \fr Z$, then we have that $\Ibar(X;Y) \geq \Ibar(X;Z)$. (DPI for Continuous Maps).
\item If $X \rightarrow Y \rightarrow Z$ represents a general Markov Chain, then we have that $\Ibar(X;Y) \geq \Ibar(X;Z)$. (DPI for General Markov Chains).
\item If $X\frf Y$, $\Hbar(X)=\Hbar(Y)$.
\item (Sampling Invariance) Let $X\in\R^d$ be represented as $d$ RVs $(X_1,X_2,..X_d)$. Consider functions $f_1,f_2,..f_d$ all from $\R \xrightarrow{ } \R$, such that they are order preserving and continuous ($X<Y$ implies $f_i(X)<f_i(Y)$). Then we have that, $\Hbar(X)=\Hbar(X_1,X_2,..X_d)=\Hbar(f_1(X_1),f_2(X_2),..f_d(X_d))$. We denote this property of \textit{ID-Entropy} as \textit{sampling-invariance}, where $f_1,f_2,..f_n$ represent the re-sampling functions. 
    
\end{enumerate}
\begin{remark}
We note that \textit{ID-Entropy} shares many of the desirable properties of discrete Shannon entropy. This includes P4, P5, P6, P8 and P9. In the Appendix, available in \cite{arxiv_id_entropy}, we also define a tighter variant of $\epsilon$-ID-Entropy, called $\epsilon$-log-ID-Entropy, which more smoothly connects to discrete Shannon entropy.
\end{remark}
\begin{remark}
We see that \textit{ID-Entropy} is sampling-invariant (P11), thus $\Hbar(X)$, does not change when the individual 1-d RVs within $X$ undergo separate order-preserving sampling functions. This is a desirable when $X$ represents real data sampled through sensors, as this ensures that $\Hbar(X)$ does not change in response to different choices of sensors with slightly different sensitivities, as long as they are order preserving. For instance, for vision sensors, one can potentially have many possible ways by which the pixels respond to brightness changes. As long as lower brightness gets mapped to a lower response, the \textit{ID-Entropy} of the sampled signal stays the same.
\end{remark}


\begin{remark}
We refer to our initial discussion regarding the observation that $\R^d$ and $\R$ are equivalent in terms of cardinality, thus $\R^d$ and $\R$ cannot be differentiated cardinality-wise. However, as $\R$ is not homeomorphic to $\R^d$, only via their topological properties we can differentiate $\R^d$ and $\R$. If we consider distributions $P(X)$ dense in $\R^d$, then it will have an \textit{ID-Entropy} of $d$, as any dense distribution is homeomorphic to $\R^d$. Applying this to the case of $d=1$, we see that the ID-Entropies of all 1-D distributions with dense support is $1$, contrary to differential entropy, which heavily depends on the shape and statistics of the distribution. 
\end{remark}



\subsection{Connection to Differential Entropy}
We show that there exists scenarios where \textit{ID-Entropy} is related to the traditional differential entropy \cite{cov_thom}. 

\begin{definition}
(\textbf{Differential Entropy})
The differential entropy $H(X)$ of a RV $X$ with an underlying probability density function of $P(X)$ is defined as
$H(X) = \E_{X} \left[\log(P(X))\right]$.
\end{definition}

The result, shown in Proposition \ref{prop1}, finds that only when one re-imagines the notion of \textit{probability density}, the two metrics can be shown to be of similar form. 

\begin{prop}\label{prop1}
Let us consider distributions $P(X)$ such that for all points $X_0 \in \R^d$ which lie in the support of $P(X)$, and for  $\epsilon \xrightarrow[]{} 0^{+}$ the following holds. 
\begin{equation}
    \int_{\lVert X-X_0 \rVert \leq \epsilon} P(X)dX = F(\epsilon)
\end{equation}
where $F(\epsilon)$ is some function of $\epsilon$. The above imposes a constraint where the distribution is evenly distributed across the support of $P$. Next, we consider a slightly modified notion of probability density $P_{supp}(X)$, where $P_{supp}(X_0)=\lim_{vol(V)\rightarrow 0}\int_{X\in V} P(X)dX/vol(V)$, where $V$ is the $k$-dimensional hypercube around $X_0$ and $k$ is the ID at $X_0$. Thus $P_{supp}(X)$ measures how the probability is distributed w.r.t the local distribution support. Then, as $\epsilon\xrightarrow[]{} 0^{+}$,  
\begin{equation}
    \E_{X} \left[\log(P_{supp}(X))\right] = \log(\epsilon)\Hbar(X) +  \log \frac{K}{F(\epsilon)},
\end{equation}
for some fixed scalar $K$. 
\end{prop}

\begin{remark}
Note that there are two separate conditions for the result in Proposition \ref{prop1} to hold, (a) the equi-probability condition indirectly enforces that among all distributions which are only non-zero valued in the support of $P(X)$, $P(X)$ is of the largest differential entropy, and (b) the re-imagining of probability density with respect to the volume of the support of $P(X)$, rather than the hypercube in $\R^d$, which is usually the case. For example, if the data manifold in $\R^d$ is a 2D plane, the probability density is the total probability divided by the total surface area of the plane. 
\end{remark}

\section{Perspectives on ID-Entropy}

In this section, we provide additional results pertaining to \textit{ID-Entropy} and its variants, to shed more perspectives on the nature of information encoded by the metric.

\subsection{ID-Entropy: New Bottleneck Criteria}
We find that \textit{ID-Entropy} can be used to formulate information bottleneck criteria similar to \cite{tishby_bottle}, for both auto-encoders and classifiers. Let $T$ represent the hidden layer of a network. $\Hbar(T)$ is non-zero and finite, unlike $I(X;T)$ which is actually infinite when $T$ is a deterministic function of $X$. We outline the two bottleneck criteria in the following definitions, which are later also supported by the theoretical results in the following section.




\begin{definition}\label{prop3}
(\textbf{Bottleneck for Auto-Encoders}) Consider the class of auto-encoders which can be represented sequentially as $X \fr T \fr \tilde{X}$, where $T$ is the encoding of $X$ as some latent feature representation. When $X=\tilde{X}$, we have that $  \Hbar(X)=\Hbar(T)\leq \Hbar^{\epsilon}(T) \leq dim(T)$.
Thus, we outline the condition for a \textit{relaxed} information bottleneck criterion as follows.
\begin{equation}
    \min \Hbar^{\epsilon}(T) \ \mbox{ s.t. } \  \tilde{X}=X.
\end{equation}

\end{definition}

\begin{definition}\label{prop4}
(\textbf{Bottleneck for Classifiers}) Consider the class of feedforward supervised architectures which can be represented sequentially as $X \fr T \fr \tilde{Y}$, where $T$ represents any hidden layer within the architecture. Given this, we outline the conditions for a \textit{relaxed} information bottleneck criterion as follows.
\begin{equation}
    \min \Hbar^{\epsilon}(T) \ \mbox{ s.t. } \ \tilde{Y}=Y.
\end{equation}
\end{definition}

\begin{remark}
Note that unlike conventional bottleneck measures which work with the MI between the features and inputs, or features and labels, here, the bottleneck principle for ID-Entropy, in both scenarios, concerns with the $\epsilon$-ID-Entropy of the features. Specifically, for these bottlenecks, the objective is to minimize $\epsilon$-ID-Entropy of the features, under the condition that the network has a perfect fit on the output, for both scenarios. 
\end{remark}

\subsection{ID-Entropy: Relevance to Generalization} \label{sec:generalization}
In this section, we provide two theoretical results that connect ID-Entropy of the input and the hidden layers to generalization error, for both auto-encoders and classifiers. We first give some necessary definitions, after which we provide the bounds on generalization error for both cases.

\begin{definition}
(\textbf{Label Generating Function (LGF):}) For any given binary classification task, where the input data $X\in \R^d$ and $X\sim P$, and labels $Y\in \{-1,1\}$, we define the ground truth label generating function $g_{true}$ as the function which generates the true labels on all datapoints $X\sim P$. 
\end{definition}

\begin{definition}
(\textbf{Probabilistically Lipschitz:}) Any function $f:\R^d\xrightarrow{} \{-1,1\}$ is said to be probabilistically $C_{p}$-Lipschitz, if it satisfies the following.
For any $X,X'\in R^d$, 
\begin{equation}
    P\left(f(X)\neq f(X')\right) \leq C_{p}\lVert X-X' \rVert. 
\end{equation}
\end{definition}

\begin{theorem}  \label{thm:unsup}
\textbf{(Auto-Encoders)} We consider a scenario with an auto-encoder network.  Let $S=\{X_1,X_2,..,X_m\}$, be the training data points. Also, let $\widehat{err}_{MSE}(S)$ represent the training mean-squared reconstruction error and $err_{MSE}(g(f_T(X)),X)=\mathbb{E}_{S}[\widehat{err}_{MSE}(S)]$ be the generalization error, where $f_T(X)$ represents the function at any layer $T$, and $g(.)$ represents the reconstruction function. Let $g_{true}(.)$ be the ground truth reconstruction function. We consider architectures which yield continuous $f_T$ and $g$. Given that $X\sim P$, let $d_{max}=\max_{X\sim P,i} \lVert X-X_i \rVert$. Then, as $m\xrightarrow[]{} \infty$, for Lipschitz continuous $g$, $g_{true}$ and $f_T$, and for some $C>0$ we have that
\begin{equation}
    err_{MSE}(g(f_T(X)),X) \leq \widehat{err}_{MSE}(S) + \frac{C d_{max}^2\Hbar(T)}{\Hbar(T)+2},
\end{equation}
where $C$ depends on Lipschitz constants of $g$, $g_{true}$ and $f_T$. 
\end{theorem}


\begin{corollary} \label{corr:sup}
\textbf{(Classifiers)}  We are given training data points and labels $S=\{(X_1,y_1),(X_2,y_2),..,(X_m,y_m)\}$.  Also, let $\widehat{err}(S)$ represent the 0-1 loss on the training data and $err(g(f_T(X)),y)$ represent the generalization error, where $f_T(X)$ represents the function at any layer $T$, and $g(f_T(X))$ yields the network output.  Let $g_{true}(.)$ be the ground truth label generating function. We consider architectures which yield continuous $f_T$ and $g$. Then, as $m\xrightarrow[]{} \infty$, for probabilistically Lipschitz continuous $g$ and $g_{true}$ and Lipschitz continuous $f_T$, we have that
\begin{equation}
    err(g(f_T(X)),y) \leq \widehat{err}(S) + \frac{C_{p} d_{max}\Hbar(T)}{\Hbar(T)+1},
\end{equation}
where $C_{p}$ depends on the probabilistically Lipschitz constants of $g$ and $g_{true}$, and the Lipschitz constant of $f_T$. 
\end{corollary}

\begin{remark}
Note that these theoretical results support the bottleneck criteria defined in the previous section. Taken together, Theorem \ref{thm:unsup} and Corollary \ref{corr:sup} both indicate that a larger ID-Entropy of the feature layers will likely lead to a larger generalization gap. Note that these results are derived for the asymptotic case of infinite training examples, but we should expect them to hold for large training datasets as well, and to a certain extent in smaller training datasets. In our experiments we indeed find that for moderately sized training datasets in MNIST and CIFAR-10, ID-Entropy of feature layers indeed grow with the generalization gap.  
\end{remark}

\begin{remark}
Note that these results are not intended to be computable bounds, but rather to provide insights into the overall trends on how ID-Entropy can potentially relate to the generalization gap, particularly the asymptotic behaviour of the gap. Note, as $m\rightarrow \infty$, we have $d_{max}\rightarrow 0$.
\end{remark}

\subsection{ID-Entropy: Connections to Causality}
The properties of \textit{ID-Entropy} indicate that it could potentially encode the underlying causal factors under a specific set of assumptions for cause-effect relationships. The following proposition shows this for the case where $k$ independent, continuous valued causes generate the data. 

\begin{prop}\label{prop2}
Consider any arbitrary causal diagram (directed acyclic graph) which has an observable RV $X\in \R^d$, and a set of hidden variables (potentially causes) in $C = \{C_1,C_2,...,C_n\}$ where $C_i\in \R \ \forall i$. Furthermore, assume that if $(X_1,X_2,...,X_p)$ are the parent nodes of $Y$, then we have that $(X_1,X_2,...,X_p)\fr Y$. Given this, we have
\begin{equation}
    \Hbar(X) \leq \min k \  \mbox{ s.t. } \  (C_{a(1)},C_{a(2)},..C_{a(k)}) \fr X, 
\end{equation}
where $1 \leq a(i)\leq n$. Thus, the \textit{ID-Entropy} of $X$ is less than the minimum number of variables in $C$ such that $X$ is f-related to them.  
\end{prop}

\begin{remark}
The result in proposition \ref{prop2} holds for causal models where, if $X_1,..,X_k$ are the sole causes of $Y$, then $(X_1,..,X_k)\fr Y$, i.e., $Y$ is some continuous function of $X_1,..,X_k$. In this setting, we find that the ID-Entropy of observable variables are internally related to the minimal number of causes that functionally relate to the observed variable. As ID-Entropy is invariant to the shape and statistics of the distribution (Remark 3), in the context of proposition \ref{prop2} it implies that \textit{ID-Entropy} can capture  \textit{structural} information pertaining to the hidden causes, rather than purely statistical information. 
\end{remark}

\section{Related Work}
With regard to our proposed variants of ID-entropy measure, their various entropy-like properties and the relevance to generalization, we did not find much directly related work. However, there are various bodies of work in literature which have individually studied the various aspects of our proposed measure. First, we note that as one of our objectives is to find a suitable measure of entropy that is invariant to homeomorphisms, this rules out the well-known variants of divergence measures in the literature. They include, KL-divergence, Bregman-divergence, and the f-divergences \cite{div_stat}. We note that these measures are not topological invariants, i.e., they are not invariant to the set of all homeomorphisms. Similarly, other topological properties of a distribution such as the Hausdorff dimension and the fractal dimensions are not topological invariants \cite{frac_s}. We also note that intrinsic dimensionality has been studied in previous works \cite{id_1,id_2}, not from the perspective of an entropy measure, however. Rather, the use of homeomorphisms to $\R^n$ in our work in order to arrive at the local dimensionality is more aligned with the concept of topological dimension \cite{top_dim}. However, the topological dimension is a global metric whereas our proposed measures are expected values of their local counterparts. There has been a body of work which connects dimensionality and entropy \cite{renyi_infodim,frac_ent}, however these approaches are different to our formulation of entropy, and can be characterized through the fractal dimension of the distribution, which are not topological invariants. This includes the information dimension \cite{infordim_orig}, which was found to relate to ID under a similar expression in \cite{id_infodim}, only under the finite entropy constraint, thus not being a true topological invariant. Apart from the aforementioned, there are works that propose alternative entropy definitions motivated by the structure of data \cite{datastruct_ent_1,datastruct_ent_2,datastruct_ent_3}, but are not ID-based and are not topological invariants. Note that here we define ID in a specific manner, encompassing both discrete and continuous variables \eqref{eq:master}, that extends to distributions that aren't locally Euclidean (e.g., intersecting manifolds). Lastly, we additionally discuss its relevance to a broad range of concepts such as causality, information bottlenecks and generalization error. With regard to generalization error, recent studies in \cite{id_neur,id_1,id_2,id_genquad} also note that low data ID yields a lower generalization error. However, (i) these studies mainly consider the global ID and (ii) a theoretical study of how the expected local ID asymptotically relates to generalization error for both classifier and auto-encoder architectures is lacking. Other work also has studied the relevance of ID (global) to generalization error such as in \cite{id_1}, however, the ID there is estimated for the network function, rather than the input and hidden layers. 
\begin{figure*}[t]
\begin{subfigure}[b]{.33\linewidth}
  \includegraphics[width=\linewidth]{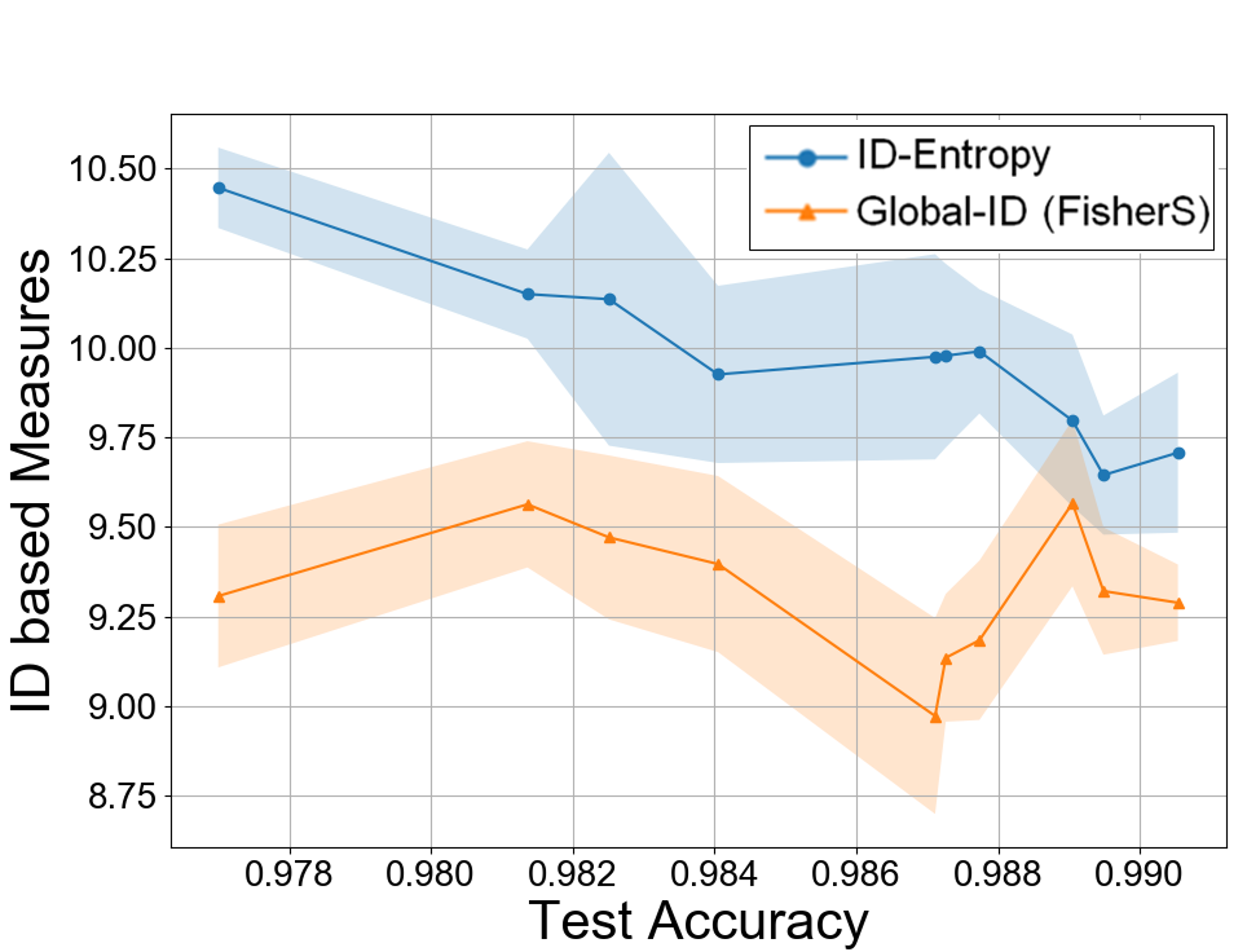}  
  \caption{MNIST, Classifiers}
\end{subfigure}
\begin{subfigure}[b]{.33\linewidth}
  \includegraphics[width=\linewidth]{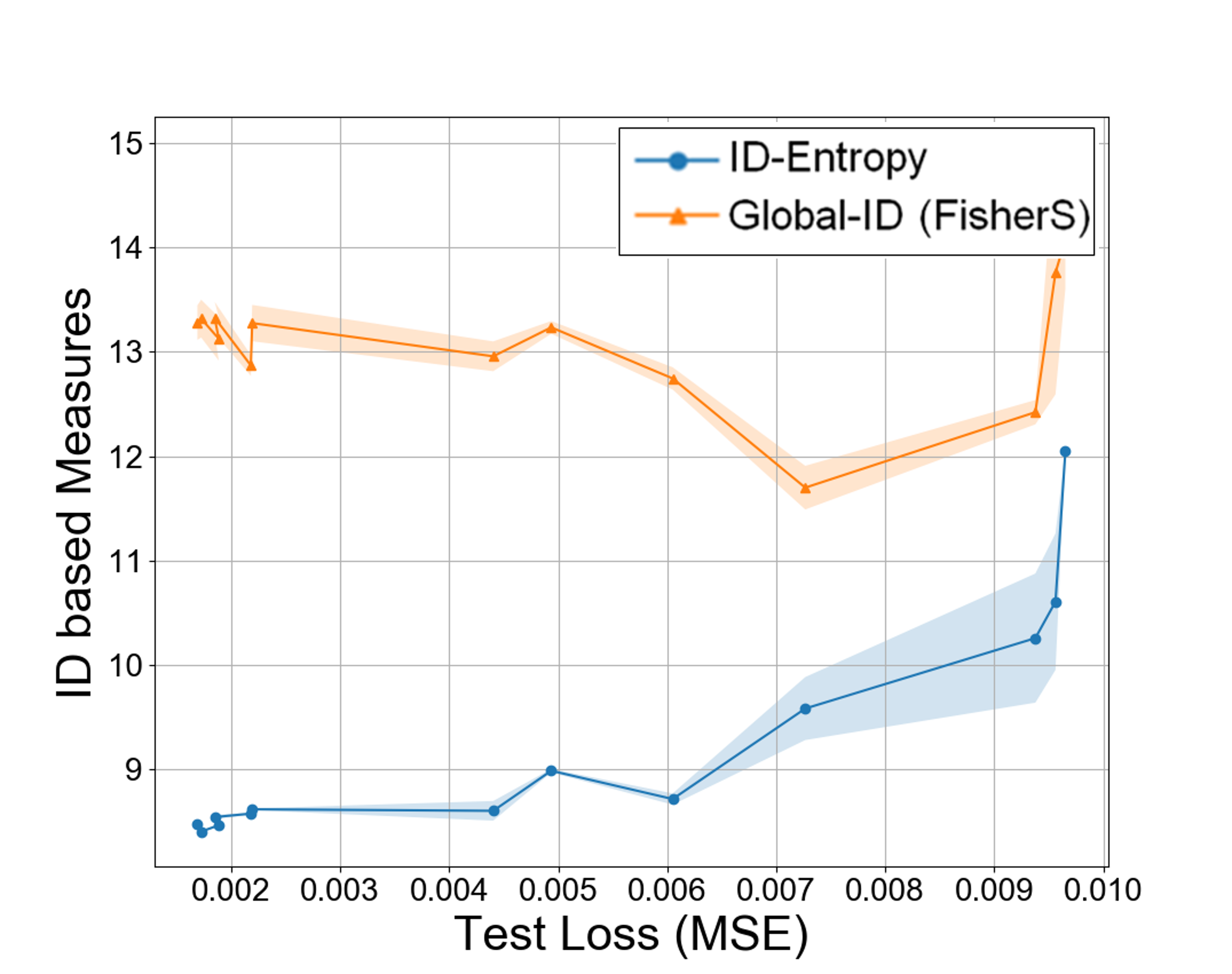}  
  \caption{MNIST, Auto-Encoders}
\end{subfigure}
\begin{subfigure}[b]{.33\linewidth}
  \includegraphics[width=\linewidth]{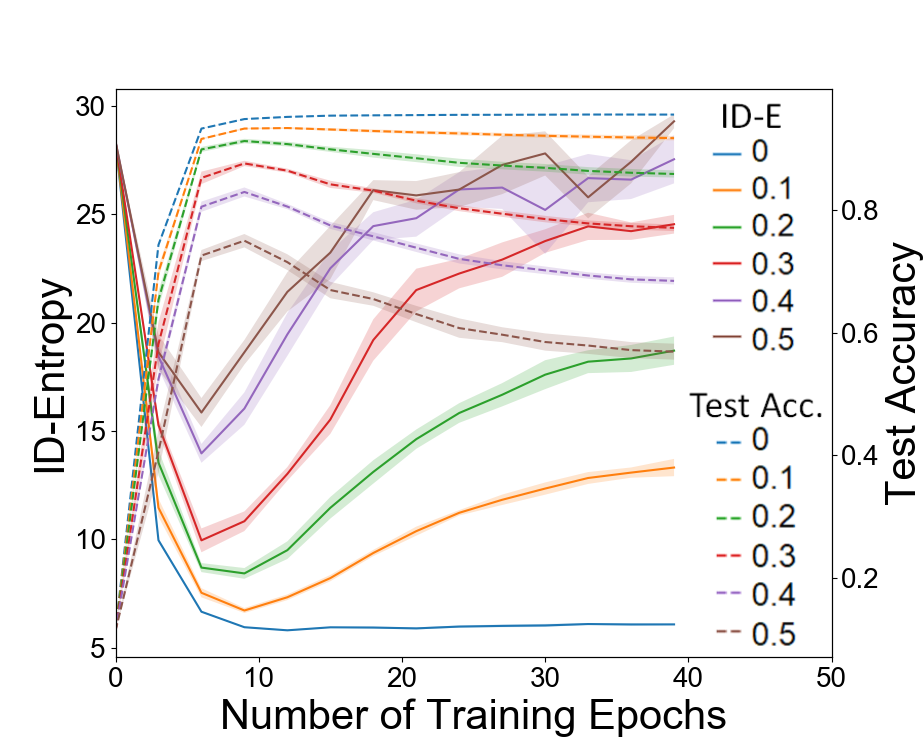}  
  \caption{MNIST, Label Noise}
\end{subfigure}
\begin{subfigure}[b]{.33\linewidth}
  \includegraphics[width=\linewidth]{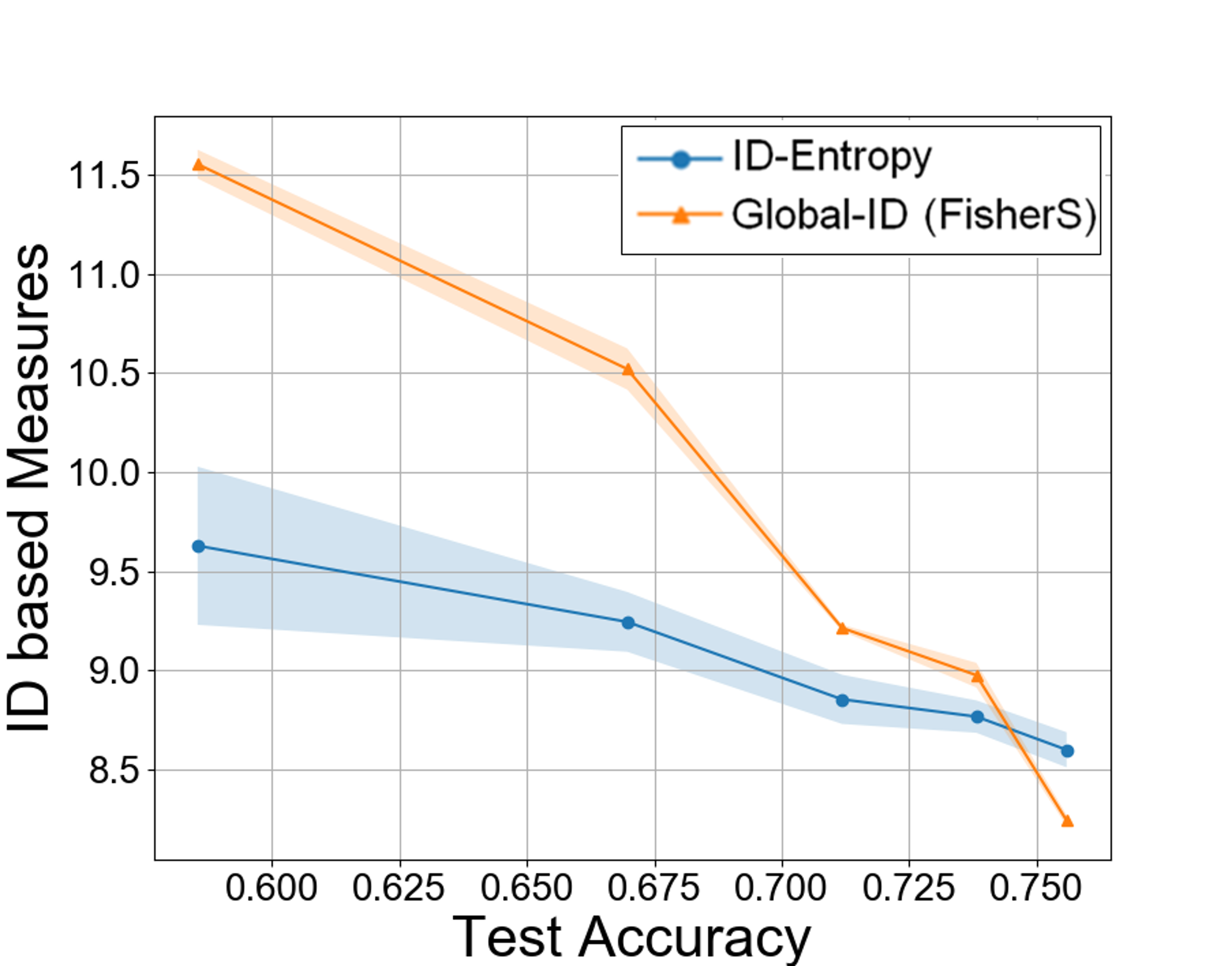}  
  \caption{CIFAR-10, Classifiers}
\end{subfigure}
\begin{subfigure}[b]{.33\linewidth}
  \includegraphics[width=\linewidth]{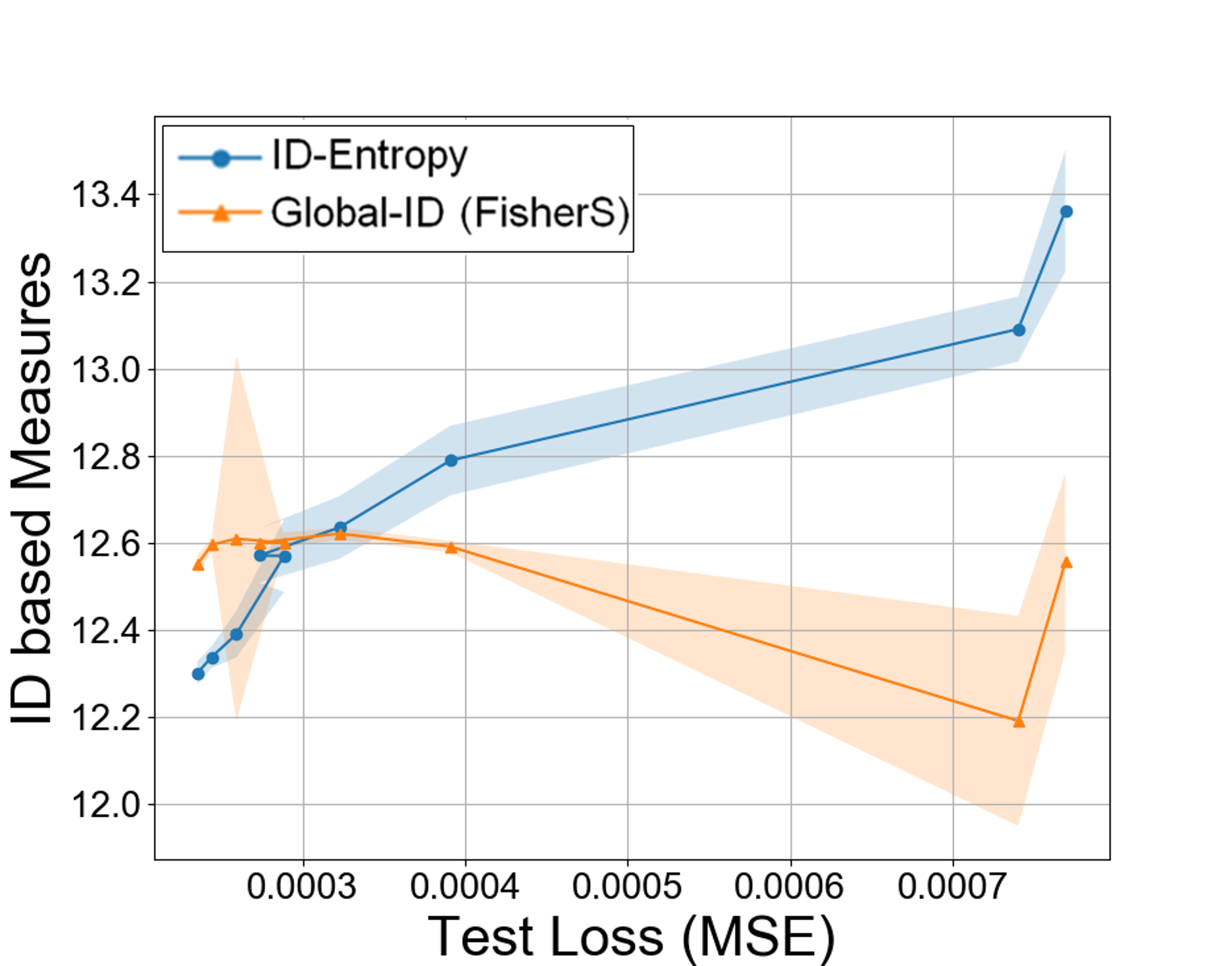}  
  \caption{CIFAR-10, Auto-Encoders}
\end{subfigure}
\begin{subfigure}[b]{.33\linewidth}
  \includegraphics[width=\linewidth]{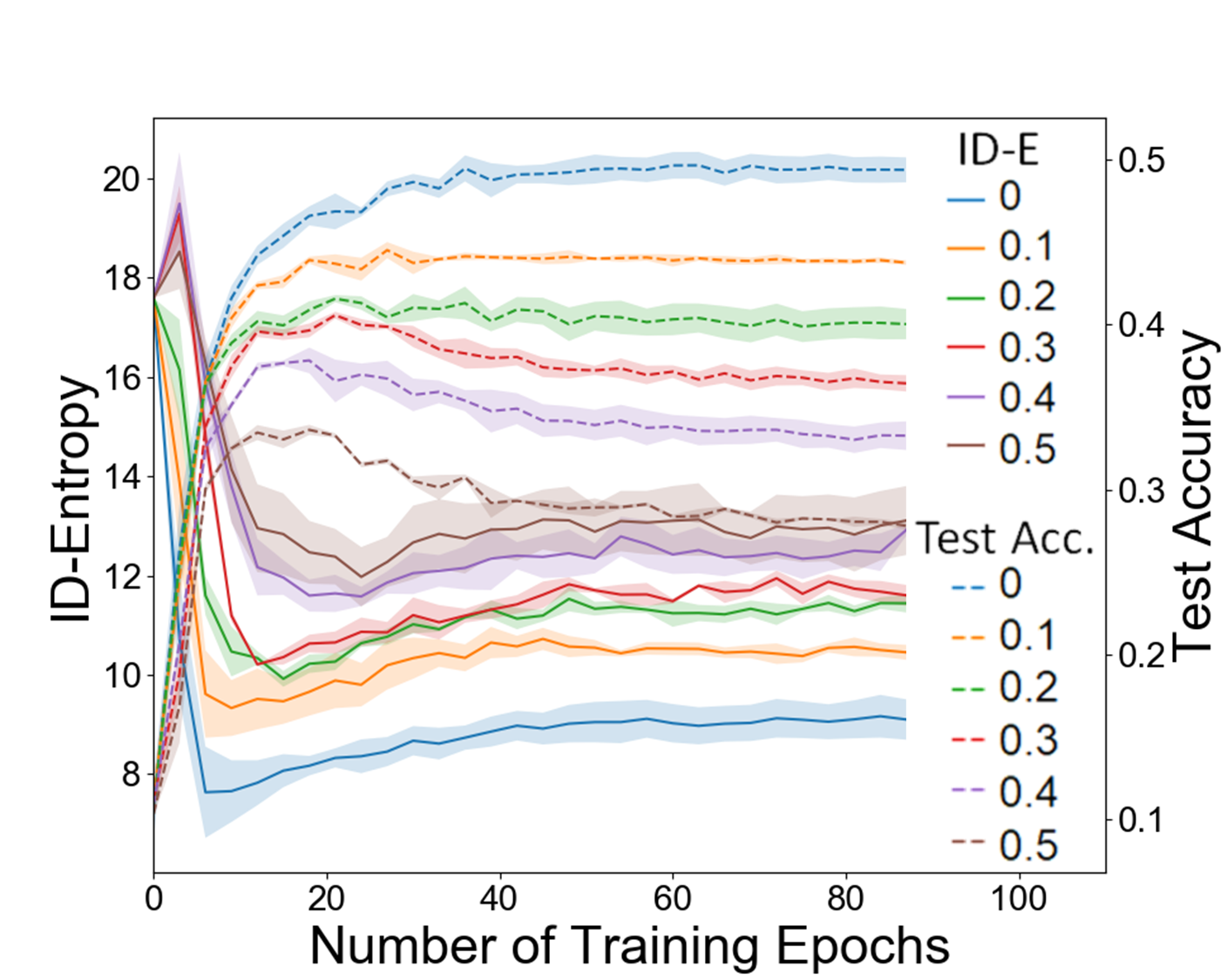}  
  \caption{CIFAR-10, Label Noise}
\end{subfigure}
\caption{Figures showing the various trends of ID-Entropy and the model performance, for models trained on MNIST and CIFAR-10. (a) and (d) show the average ID-Entropy and global ID (scaled) of the last hidden layer for a 4-layer CNN trained on MNIST and a Resnet-44 trained on CIFAR-10, for various choices of training data size, each leading to a different average test accuracy, (b) and (e) show the dependence between test loss (mean-squared error) and the ID-Entropy and global ID of the encoded feature layer, for a 4 layer Convolutional auto-encoder trained on MNIST and CIFAR-10, and (c) and (f) show the progression of average ID-Entropy of the last hidden layer for a 4-layer CNN trained on MNIST and a ResNet-44 trained on CIFAR-10 for different levels of label-noise in the training samples (from 0 to 0.5).}
\label{fig:all_figs}
\end{figure*}

\section{Experiments and Discussions}
Here we showcase two experiments, where we compute the ID-Entropy of the feature layers of \textit{classifier} and \textit{auto-encoder} architectures on MNIST and CIFAR-10, and contrast it with generalization performance. For all experiments we estimated the ID-entropy using Algorithm \ref{alg:one}, with $n=2000$ and $k=100$.  We used FisherS intrinsic dimensionality estimator in \cite{local_id_survey} as the global ID estimator $f_{ID}$, as we found it to be a robust choice. Following the results of Theorem \ref{thm:unsup} and Corollary \ref{corr:sup}, higher ID should point to worse generalization performance and vice-versa. Experiment details and further results are given in the Appendix, available in \cite{arxiv_id_entropy}, and code is available at: https://github.com/kentridgeai/ID-Entropy.
\subsection{ID-Entropy Computation: Classifiers} 

\textbf{Generalization Gap and ID-Entropy:} With a fixed 4-layer CNN architecture for MNIST and a ResNet-44 for CIFAR-10, we repeat the training routine with different choices of the training data size and random network initializations. This leads to significant changes in test accuracy across varying training data sizes, with the training accuracy of the resulting networks being $\approx$ 100\% in each case. Subsequently, we report the average ID-Entropy for each choice of training dataset size across 5 separate runs and their standard deviation, and plot them alongside the average test accuracy for the corresponding training dataset size in Figure \ref{fig:all_figs} (a). Furthermore, we also report the corresponding estimates of global ID using the FisherS estimator in each case. We see that ID-Entropy of the feature layer clearly reduces as the test accuracy increases, in agreement with the implications of Corollary \ref{corr:sup}, whereas we find that the global ID also yields similar trends for CIFAR-10, but not for MNIST. 

\noindent\textbf{Label-Noise Addition:} For both MNIST (4-layer CNN) and CIFAR-10 (ResNet-44), we add label noise by randomly changing the label of a training data point with a certain probability. It is well known that adding label noise makes the problem harder \cite{noi_d}, significantly impacting test performance, increasing the generalization gap. We summarize the results in Figure \ref{fig:all_figs} (c), where we plot the trajectories of average training error, average test error, and average ID-Entropy of the CNN feature layer (second to last), for various choices of the label noise probability $p=\{0,0.1,0.2,0.3,0.4,0.5\}$. All reported measures are computed by averaging over five runs. The results show that with greater label noise, the test accuracy declines and the ID-Entropy of the features increases. Furthermore, in each case, we also find that the training time trend of ID-Entropy follows an opposite trend compared to the network's test performance in most cases. Lastly, we see that for the no-noise case $p=0$, ID-Entropy shows a sharp decline, after which it either stabilizes or slowly increases with training.

\noindent\textbf{Information Bottleneck Discussion:} With respect to our proposed bottleneck principle, the ID-Entropy of the last hidden layer $T$, $\Hbar(T)$, is also the mutual information $\Ibar(X;T)$, due to P7. As our proposed bottleneck principle dictates that ideally $\Hbar(T)$ must be minimized, we verify this in our experiments. As the results in Figure 1 (c) and 1 (f) show, $\Hbar(T)$ has a consistent trend during training. Furthermore, in contrast to $I(X;T)$ in \cite{tishby_bottle}, which predicts two distinct phases in the respective order: fitting ($I(X;T)$ and $I(Y;T)$ increase) and compression ($I(X;T)$ decreases), we find that the phases are \textit{reversed} for $\Ibar(X;T)$. From the perspective of ID-Entropy, as seen in the label noise plots, the network can undergo up to two distinct phases: (i) \textit{simultaneous compression and fitting} (ID-Entropy decreases) and (ii) \textit{decompression} (ID-Entropy increases), and only in the latter phase the network potentially can overfit the dataset. This agrees with observations in \cite{iclr_2021_geom_NN} which finds that networks learn generalizable features in early epochs and only memorize later. When $p=0$ (no label noise), we find that the network simply minimizes $\Hbar(T)$ as training progresses, and thus the decompression phase isn't observed.

\subsection{ID-Entropy Computation: Auto-Encoders}
Like before, we plot the ID-Entropy of a 4-layer Convolutional auto-encoder (CAE) trained on the MNIST and CIFAR-10 datasets. In Figure \ref{fig:all_figs} (b) and (e), we plot the ID-Entropy and global ID of the feature layer, which is also the input layer for the decoder part of the CAE, and plot the average test loss (mean-squared error) during training. All reported measures are computed by averaging over five runs. Here we find that ID-Entropy correlates to the observed test loss in both datasets, which agrees with the implications of the Corollary \ref{thm:unsup}. Furthermore, in both cases, we find that the global ID does not adhere to any such clear trends. 

\vspace{-1mm}

\section{Reflections}
Through its properties, we find that the proposed measure of \textit{ID-Entropy} encodes more \textit{structural} information regarding the distribution, rather than statistical. Furthermore, the main differentiator between \textit{ID-Entropy} and conventional differential entropy is that \textit{ID-Entropy} preserves the intrinsic dimensionality information of the data manifold, whereas, for differential entropy and other cardinality-driven variants, this information can be lost, due to Cantor's argument.   

\textit{ID-Entropy} also seems to relate to causal factors that may be responsible for generating the data (as seen in Proposition \ref{prop2}), when we assume that cause-effect relationships are only via continuous functions, which yield f-relations between causes and effects. As such, f-relatedness also represents a fundamental property between two variables, and is independent of the probability distributions of the causal variables, i.e., changing the distribution of the cause $X$ through interventions still preserves the property that $X\fr Y$. Also, as in the definition of \textit{ID-Entropy}, note that only when $X\frf Y$, $X$ and $Y$ can be said to be homeomorphic. 
 

The results relating ID-Entropy to generalization showcase the relevance of ID-Entropy to generalization error both in the context of classifier and auto-encoder architectures. Empirically we see this to be the case as well. An interesting empirical finding was that for datasets that are primarily \textit{harder} to generalize, such as the ones with label-noise tested here, the ID-Entropy of the hidden layers shows a significant increase. Thus, to a certain extent, ID-Entropy showcases the level of memorization in the features learnt by the network, as noisy labels forces a neural network to memorize the labels, leading to poor test performance \cite{genun}.

\section*{Acknowledgements}
This research is supported by A*STAR, CISCO Systems (USA) Pte. Ltd and the National University of Singapore under its Cisco-NUS Accelerated Digital Economy Corporate Laboratory (Award I21001E0002). 
Additionally, we would like to thank the members of the Kent-Ridge AI research group at the National University of Singapore for helpful feedback and interesting discussions on this work.

\bibliography{egbib}


\newpage
\onecolumn
\appendix

\section*{\LARGE Appendix}

In this appendix, we provide the proofs of all theoretical results and properties of ID-Entropy mentioned in the main paper. Furthermore, we provide additional experimental details for all experiments conducted in the main paper. We also define a tighter version of $\epsilon$-ID-Entropy proposed in the main paper, called  $\epsilon$-log-ID-Entropy.

\section{Additional Results}
For all experiments, the associated architectures used to obtain the results in Section 6 of the main paper, is shown in Table \ref{tab:ars}. The headings of the following sections refer to the same experiment as their corresponding parts with the same headings in Section 6 of the main paper. For each experiment, all ID-Entropy computations were done on a fixed size validation set with a fixed $k$-nearest neighborhood. Note that this removes any effect that the choice of number of data samples has on the estimated ID. 
\subsection{MNIST}
\textbf{Generalization Gap and ID-Entropy}: For this experiment, we varied the number of training data samples roughly evenly between 5000 and 35000. For each size of training dataset, the networks were trained for 80 epochs with random initializations, and all reached 100\% training accuracy. All numbers are reported after 5-fold averaging.
\begin{table}[hb] 
\centering
\renewcommand{\arraystretch}{1.2}
\begin{tabular}{|c|c|c|}
\hline
{Dataset}                   & {Scenario}                                                               & {Architecture Flow}                                                       \\ \hline
\multirow{2}{*}{MNIST}    & \begin{tabular}[c]{@{}c@{}}Classifier \\ (CNN) \end{tabular}   & 25C(3x3)-2M-45C(3x3)-2M-60C(3x3)-7M-50F-10                              \\ \cline{2-3} 
                          & \begin{tabular}[c]{@{}c@{}}Convolutional \\ Auto-Encoder \end{tabular} & 16C(3x3)-2M-4C(3x3)-2M-16TC(2x2)(2,2)-1TC(2x2)(2,2)                     \\ \hline
\multirow{2}{*}{CIFAR-10} & \begin{tabular}[c]{@{}c@{}}Classifier \\ (CNN) \end{tabular}   & 16C(3x3)-7x(16RN-Block)-7x(32RN-Block)-7x(64RN-Block)-AP-10 \\ \cline{2-3} 
                          & \begin{tabular}[c]{@{}c@{}}Convolutional \\ Auto-Encoder  \end{tabular} & 24C(4x4)(2,2)-2M-48C(4x4)(2,2)-2M-24TC(4x4)(2,2)-3TC(4x4)(2,2)          \\ \hline
\end{tabular}%
\caption{Table showing the various architecture flows used in the experiments on MNIST and CIFAR-10 reoprted in the main paper. We describe the details of the layer representation as follows. (a) $k$M: ($k\times k$) Max-Pooling, (b) $m$C($k\times k$)($s,s$): ($k\times k$) Convolution with stride ($s,s$) ($s=1$ if not mentioned), yielding $m$ feature maps, (c) $m$TC($k\times k$)($s,s$): ($k\times k$) Transposed Convolution with stride ($s,s$) ($s=1$ if not mentioned), (d)AP: Average-Pooling, (e)$n$F: Fully connected layer with $n$ output nodes  (f) $n$x($m$RN-Block): $n$ ResNet Blocks each with $m$ output units as described in \cite{hers}.}
\label{tab:ars}
\end{table}

\textbf{Label-Noise:}
Five sets of training data of a fixed size of 2000 were created, and in each case a different labelling was generated by randomly changing a fraction $p$ of the original labels. The resulting trajectories of the ID-Entropy and Test Accuracy were plotted and the averages and deviations were reported in the main paper. Note that all networks here have the same initialization. This was then repeated for 6 different choices of $p$, $p=(0,0.1,0.2,0.3,0.4,0.5)$. 

\textbf{ID-Entropy Computation: Auto-Encoders:}
Here, we varied the number of training examples roughly evenly between 500 and 30,000, and in each case the average test accuracies are reported after 50 epochs of training with randomly initialized networks. All reported results are with five-fold averaging. 

\subsection{CIFAR-10}
\textbf{Generalization Gap and ID-Entropy}: We varied the number of training data samples roughly evenly between 20,000 and 30,000. For all choices of training dataset size, the networks were trained for 60 epochs with random initializations, and achieved similar training accuracies and significantly different test accuracies. All numbers are reported after 5-fold averaging.

\textbf{Label-Noise:} 
Five sets of training data of a fixed size of 5000 was created, and in each case a different labelling was generated by randomly changing a fraction ($p$) of the original labels. The resulting trajectories of the ID-Entropy and Test Accuracy were plotted and the averages and deviations were reported in the main paper. Note that all networks here have the same initialization. This was then repeated for 6 different choices of $p$, $p=(0,0.1,0.2,0.3,0.4,0.5)$. 

\textbf{ID-Entropy Computation: Auto-Encoders:}
Here, we varied the number of training examples roughly evenly between 250 and 5000, and in each case the average test accuracies are reported after 100 epochs of training with randomly initialized networks. All reported results are with five-fold averaging.

\subsection{Additional Discussions}
Note that in all cases, we wished to test whether ID-Entropy could be informative of the network's generalization performance, when a moderately sized training dataset was used. Note that our theoretical results hold for the limiting case of $m\arrow \infty$, but our experiments show that even for moderately sized $m$, thereby using a fraction of the total training examples, ID-Entropy can still be indicative of generalization performance to a certain extent. Furthermore, typically generalization error bounds use the number of training data examples in their estimation, and thus naturally yield smaller generalization gap as $m$ increases. However, here, as ID-Entropy is estimated from a fixed size validation set in each case, it does not use the information provided by the number of training data samples. Thus, given any trained network, its ID-Entropy can be computed without needing additional information about the number of training data samples used to train the same.

\subsection{Tighter Variant of ID-Entropy}

We refer to (1) in the main paper, which is the step that defines the \textit{$\epsilon$-Neighborhood Intrinsic Dimension}, as follows. 
\begin{equation} \label{eq:master_B}
    d_{\epsilon}(\rho) = \min n \ \mbox{ s.t. } \ (v_1,v_2,..v_n) \frf X_{\epsilon}^{\rho}
\end{equation}

We note that both $\epsilon$-ID-Entropy and ID-Entropy represent the expected value of the above dimension across the distribution, for different neighborhood sizes. Here, we show that a tighter, and more intuitive version of $\epsilon$-ID-Entropy can be formulated as follows. We define the \textit{$\epsilon$-log-Neighborhood Intrinsic Dimension} as:

\begin{equation} \label{eq:master_C}
    d_{\epsilon}^{\log}(\rho) = \min \left ( n + H(v_{discrete}) \right)   \ \mbox{ s.t. } \ (v_1,v_2,..v_n,v_{discrete}) \frf X_{\epsilon}^{\rho},
\end{equation}
where $v_1,v_2,..v_n$ represent continuous 1-D RVs, $v_{discrete}$ is a discrete RV which is drawn from a set of any arbitrary cardinality, and $H$ represents the discrete entropy operator. Note that $d_{\epsilon}^{\log}(\rho)\leq d_{\epsilon}(\rho)$.

Using the above, we can define $\epsilon$-log-ID-Entropy as follows.
\begin{definition}
(\textbf{$\epsilon$-log-ID-Entropy}). We define the $\epsilon$-log-ID-Entropy $\Hbar^{\epsilon}_{\log}(X)$ of a RV $X$ with underlying distribution $P(X)$ as
\begin{equation}
    \Hbar^{\epsilon}_{\log}(X) = \E_{\rho\sim P(X)}\left [d_{\epsilon}^{\log}(\rho)\right].
\end{equation}
\end{definition}

Note that $\epsilon$-log-ID-Entropy is a tighter version of $\epsilon$-ID-Entropy, as we have that $\Hbar^{\epsilon}_{\log}(X)\leq\Hbar^{\epsilon}(X)$. We consider $\epsilon$-log-ID-Entropy to be a more intuitive variant of ID-Entropy as it more readily connects to discrete Shannon entropy in the case when the distribution over $X$ behaves like a discrete distribution. In fact, it is trivial to show that when $X$ is discrete, i.e. the cardinality of the support of $X$ is finite, then  $\Hbar^{\epsilon}_{\log}(X)=H(X)$.  
\section{Proofs of Theoretical Results} 

\subsection{Proofs of the properties of \textit{ID-Entropy}}

The proofs are provided below for each property of ID-Entropy and its extensions described in Section 3.2 of the main paper.
\begin{enumerate}[leftmargin=0cm,topsep=0cm,label={\bf P\arabic*},wide]

\item $0\leq \Hbar( X)\leq dim(X)$. 
\begin{proof}
First, note that as \textit{ID-Entropy} is the average of \textit{$\epsilon$-Neighborhood Intrinsic Dimension} ($\epsilon$-NID) which cannot be less than $0$, and thus $\Hbar(X)\geq0$. Next, consider any point $X_0$ within the support of $P(X)$. Using the definition of $\epsilon$-NID $d_{\epsilon}(X_0)$, we have that as $\epsilon \arrow 0$, there exists RVs $v_1,v_2,v...,v_{d_{\epsilon}(X_0)}$, such that 
\begin{equation}
    (v_1,v_2,v...,v_{d_{\epsilon}(X_0)}) \frf X_{\epsilon}^{X_0}.
\end{equation}
Using the locally Euclidean constraint, we have that $X_{\epsilon}^{X_0}$ is homeomorphic to $k$ for some $k\leq d$. Thus, as a symmetric f-relation implies a homeomorphism, we have from the definition of $\epsilon$-NID, that $d_{\epsilon}(X_0)\leq k\leq d$. Lastly, as $\Hbar(X)=\E_{X_0\sim P}[d_{\epsilon}(X_0)]$ as $\epsilon\arrow0$, we have that  $\Hbar(X)\leq dim(X)$. 

\end{proof}
\item $\Hbar(X)\leq \Hbar^{\epsilon}(X)$ for any $\epsilon>0$. 
\begin{proof}
Continuing with the same notation as from the proof of P1, we note that for any $\epsilon>0$, we will have that $d_{\epsilon}(X_0)\geq \lim_{\epsilon \arrow 0^{+}}d_{\epsilon}(X_0)$, as the support of the RV $\lim_{\epsilon \arrow 0^{+}}X_{\epsilon}^{X_0}$ will be a subset of the support of the RV $X_{\epsilon}^{X_0}$. The result then follows from the fact that $\Hbar^{\epsilon}(X)$ is the expected value of $d_{\epsilon}(X_0)$ where $X_0$ is generated according to the underlying data distribution. 
\end{proof}

\item $\Hbar(\alpha X)=\Hbar(X)$     \begin{proof}
The proof follows from noting that scaling $X$ would preserve the local $\epsilon$-NID, as  $(v_1,v_2,...v_k)\fr X$ implies $(v_1,v_2,...v_k)\fr \alpha X$ as well. As $\Hbar(X)$ represents the expected value of the intrinsic dimension, and as the scaling function is bijective, it remains unchanged. 
\end{proof}
\item $\Hbar( X,Y)=\Hbar(Y,X)$.\label{FEP2_B}
\begin{proof}
This result immediately follows from the observation that permuting the dimensions of $X$ does not change the intrinsic dimension at any point, as $(v_1,v_2,...v_k)\fr (X_1,X_2,..X_k)$ also implies $(v_1,v_2,...v_k)\fr (X_{(1)},X_{(2)},..X_{(k)})$, where $(X_{(1)},X_{(2)},..X_{(k)})$ represents a permutation of $(X_1,X_2,..X_k)$. Thus, $\Hbar( X,Y)=\Hbar(Y,X)$.
\end{proof}
\item $\Hbar( X,Y)\geq \Hbar( X)$ and $\Hbar( X,Y)\geq \Hbar( Y)$.  
\begin{proof}
Let us consider the $\epsilon$-NID $d_{\epsilon}(X_0)$ as $\epsilon \arrow 0^{+}$, at some $X_0\in \R^d$. Then, for all $X'\sim P_{\epsilon}^{X_0}$, there exists  $d_{\epsilon}(X_0)$ RVs $v_1,v_2,...,v_{d_{\epsilon}(X_0)}$, such that $(v_1,v_2,...,v_{d_{\epsilon}(X_0)}) \frf X'$. Now, if we consider both $X$ and $Y$ together, we have for some $Y_0\in \R^l$, the RV $Y'\sim P_{\epsilon}^{Y_0}$ as $\epsilon\arrow 0^{+}$. 

Let the $\epsilon$-NID of $(X,Y)$ at $(X_0,Y_0)$ be $d_{\epsilon}(X_0,Y_0)$. Then we have that for some  RVs $v_1,v_2,...,v_{d_{\epsilon}(X_0,Y_0)}$, $(v_1,v_2,...,v_{d_{\epsilon}(X_0,Y_0)}) \frf (X',Y')$. As $\epsilon \arrow 0^{+}$, $d_{\epsilon}(X_0,Y_0)$ would represent the dimensionality of the manifold at $(X_0,Y_0)$, and we can show using proof by contradiction that $d_{\epsilon}(X_0,Y_0)\geq d_{\epsilon}(X_0)$. 

Assume that $d_{\epsilon}(X_0,Y_0)<d_{\epsilon}(X_0)$, then in the $\epsilon$ neighborhood around $X_0$, there exists two points $A$ and $B$ which have infinite geodesic distance, as otherwise $d_{\epsilon}(X_0,Y_0)<d_{\epsilon}(X_0)$ wouldn't hold. This will violate our assumption (b) or (c), and thus yields that $d_{\epsilon}(X_0,Y_0)\geq d_{\epsilon}(X_0)$. Thus,   $\Hbar^{\epsilon}( X,Y)\geq \Hbar^{\epsilon}( X)$ as $\epsilon\arrow 0^{+}$. Similarly, it follows that $\Hbar( X,Y)\geq \Hbar( Y)$ as well. 
\end{proof}
\item $\Hbar( X,Y)\leq \Hbar( X) +  \Hbar(Y)$.
\begin{proof}
Continuing with the notation in previous proofs, we note that the $\epsilon$-NID $d_{\epsilon}(X_0,Y_0)$ for $\epsilon\arrow0^{+}$ would represent the dimensionality of the manifold at $(X_0,Y_0)$. Thus, in this case, it directly follows that $d_{\epsilon}(X_0,Y_0)\leq d_{\epsilon}(X_0)+d_{\epsilon}(Y_0)$. Taking the expectation over all $(X_0,Y_0)\sim P(X,Y)$, we obtain $\Hbar^{\epsilon}( X,Y)\leq \Hbar^{\epsilon}( X) +  \Hbar^{\epsilon}(Y)$ as $\epsilon\arrow 0^{+}$. 
\end{proof}
\item If $X\fr Y$,  $\Hbar( X,Y) = \Hbar(X)\geq \Hbar(Y)$. 
\begin{proof}
Continuing with the notation used in the previous proofs, we note that for any $X_0$, the $\epsilon$-NID $d_{\epsilon}(X_0)$ as $\epsilon \arrow 0^{+}$ yields the symmetric f-relation between $(v_1,v_2,...,v_{d_{\epsilon}(X_0)}) \frf (X')$. Now, as $X \fr Y$, we can represent $Y=f(X)$ for some continuous function $f$. We will show that $(v_1,v_2,...,v_{d_{\epsilon}(X_0)}) \frf (X',f(X'))$. 

First, note that $(v_1,v_2,...,v_{d_{\epsilon}(X_0)}) \fr (X',f(X'))$ follows trivially, as no additional RVs are required to predict $f(X')$, given $X'$. We can show that $ (X',f(X'))\fr (v_1,v_2,...,v_{d_{\epsilon}(X_0)}) $ as follows. 

Note that $ X'\fr (v_1,v_2,...,v_{d_{\epsilon}(X_0)}) $ follows directly from the fact that $(v_1,v_2,...,v_{d_{\epsilon}(X_0)}) \frf (X')$. Next, we see that any infinitesimal change to $(X',f(X'))$ can only occur through changes in $X'$, as $Y=f(X')$ is a deterministic function of $X'$. Thus, as $ X'\fr (v_1,v_2,...,v_{d_{\epsilon}(X_0)}) $, it also follows that $ (X',f(X'))\fr (v_1,v_2,...,v_{d_{\epsilon}(X_0)}) $. 
This yields $(v_1,v_2,...,v_{d_{\epsilon}(X_0)}) \fr (X',f(X'))$ as well and thus $(v_1,v_2,...,v_{d_{\epsilon}(X_0)}) \frf (X',f(X'))$ . 

Thus, we have that the $\epsilon$-NID $d_{\epsilon}(X_0,Y_0)\leq d_{\epsilon}(X_0)$ as $\epsilon\arrow 0^{+}$.   Now, from the proof of P3 we have that the $\epsilon$-NID of joint variables $(X,Y)$ at some $(X_0,Y_0)$, which is denoted by  $d_{\epsilon}(X_0,Y_0)$, satisfies $d_{\epsilon}(X_0,Y_0)\geq d_{\epsilon}(X_0)$.

This yields $d_{\epsilon}(X_0,Y_0)= d_{\epsilon}(X_0)$, and thus $\Hbar^{\epsilon}( X,Y) = \Hbar^{\epsilon}(X)$ as $\epsilon\arrow 0^{+}$. The result $ \Hbar(X)\geq  \Hbar(Y)$ directly follows from the application of P3. 
\end{proof}
\item (Data Processing Inequality). If $X \fr Y \fr Z$, then we have that $I_{f}(X;Y) \geq I_{f}(X;Z)$. 
\begin{proof}
First, we note that $X \fr Y \fr Z$ also implies $X \fr Z$. Next, observe that $I_{f}(X;Y) \geq I_{f}(X;Z)$ can be rewritten as $\Hbar(Y)-\Hbar(X,Y) \geq \Hbar(Z)-\Hbar(X,Z)$. From P7, it follows that  $\Hbar(X,Y)=\Hbar(X)$ and $\Hbar(X,Z)=\Hbar(X)$. Thus, it suffices to show that $\Hbar(Y)\geq \Hbar(Z)$. We note that as $Y \fr Z$,  P7 also yields $\Hbar(Y)\geq \Hbar(Z)$.

\end{proof}
\item (DPI for General Markov Chains). If $X \rightarrow Y \rightarrow Z$ represents a general Markov Chain, then we have that $\Ibar(X;Y) \geq \Ibar(X;Z)$. 
\begin{proof}
We note that the general Markov Chain case can be represented as a combination of two graphs, 
 $(X,\epsilon_1)\fr Y$ and $(Y,\epsilon_2) \fr X$, where $\epsilon_1$ and $\epsilon_2$ are random variables independent of $X$. Note that the new graphs are deterministic as before and the arrows represent f-relations, as the stochasticity of the general Markov Chain has now been subsumed by the noisy variables. Thus for some continuous $f_1$ and $f_2$ , we have $f_1(X,\epsilon_1)=Y$, and $f_2(Y,\epsilon_2)=Z$. To prove our result, we simply separate the intrinsic dimensionality of both $Y$ and $Z$ at any point $p$, denoted by $d_Y(p)$ and $d_Z(p)$, into two parts. They are: $d_Y(p,X)$ and $d_Z(p,X)$ which represents the contribution of $X$, while $d_Y(p,\epsilon_1)$ and $d_Z(p,[\epsilon_1,\epsilon_2])$ represent the contribution of the noisy random variables. With this, it follows that $\Ibar(X;Y)=\mathbb{E}_p[d_Y(p,X)]$ and $\Ibar(X;Z)=\mathbb{E}_p[d_Z(p,X)]$. Lastly, as $X$ is processed via two functions in a compositonal manner, it follows that $d_Y(p,X)\geq d_Z(p,X)$. Thus, we also have $\mathbb{E}_p[d_Y(p,X)]\geq \mathbb{E}_p[d_Z(p,X)]$, yielding $\Ibar(X;Y)\geq \Ibar(X;Z)$.
\end{proof}

\item If $X\frf Y$, $\Hbar(X)=\Hbar(Y)$.
\begin{proof}
$X\frf Y$ implies that $X\fr Y$ and $Y\fr X$. From P7, we have that $X\fr Y$ implies $\Hbar(Y)\leq \Hbar(X)$. Similarly, we also have that $Y\fr X$  implies $\Hbar(X)\leq \Hbar(Y)$. This yields, $\Hbar(Y)= \Hbar(X)$
\end{proof}
\item Let $X\in\R^d$ be represented as $d$ RVs $(x_1,x_2,..x_d)$. Consider functions $f_1,f_2,..f_d$ all from $\R \xrightarrow{ } \R$, such that they are order preserving and continuous ($x<y$ implies $f_i(x)<f_i(y)$). Then we have that, $\Hbar(X)=\Hbar(x_1,x_2,..x_d)=\Hbar(f_1(x_1),f_2(x_2),..f_d(x_d))$. We denote this property of \textit{ID-Entropy} as \textit{sampling-invariance}, where $f_1,f_2,..f_n$ represent the re-sampling functions.
\begin{proof}
Note than any order preserving function which is also continuous, has a continuous inverse as well. Thus, we have that for all $i$, $x_i\fr f_i(x_i)$, and subsequently if we let $X'=[f_1(x_1),f_2(x_2),..f_d(x_d)]$, we have $X \frf X'$. Using P9, we have $\Hbar(X)=\Hbar(X')$.
\end{proof}
    
\end{enumerate}

\subsection{Proofs of Propositions 1-2 and Remarks on Definitions 8 \& 9}

\setcounter{prop}{0}

\setcounter{theorem}{0}
\setcounter{lemma}{0}

\begin{prop}\label{prop1_B}
Let us consider distributions $P(X)$ such that for all points $X_0 \in \R^d$ which lie in the support of $P(X)$, and for  $\epsilon \xrightarrow[]{} 0^{+}$ the following holds. 
\begin{equation}
    \int_{\lVert X-X_0 \rVert \leq \epsilon} P(X)dX = C(\epsilon)
\end{equation}
where $C(\epsilon)$ is some function of $\epsilon$. The above imposes a constraint where the distribution is evenly distributed across the support of $P$. Next, we consider a slightly modified notion of probability density $P_{supp}(X)$, which measures how the probability is distributed within the local support of the distribution. Then, we have for $\epsilon\xrightarrow[]{} 0^{+}$,  
\begin{equation}
    \E_{X} \left[\log(P_{supp}(X))\right] = \log(\epsilon)\Hbar(X) +  \log \frac{K}{C(\epsilon)},
\end{equation}
for some fixed scalar $K$. 
\end{prop}

\begin{proof}
First, we observe that as $\epsilon\arrow 0^{+}$ the volume of the support of the distribution within the ball defined by $\lVert X-X_0 \rVert \leq \epsilon$ would depend on the $\epsilon$-NID $d_{\epsilon}(X_0)$ at $X_0$. Specifically, note that the volume of the support of the distribution within $\lVert X-X_0 \rVert \leq \epsilon$ will converge to a scalar $K\epsilon^{d_{\epsilon}(X_0)}$ as $\epsilon\arrow0^{+}$, as $d_{\epsilon}(X_0)$ represents the intrinsic dimensionality of the manifold at $X_0$. The result follows from this observation, as we can write for $\epsilon\arrow 0^{+}$: 
\begin{align}
    \E_{X_0} \left[\log(P_{supp}(X_0))\right] &= \E_{X_0} \left[-\log\left ( \frac{\int_{\lVert X-X_0 \rVert \leq \epsilon,} P(X)dX}{K\epsilon^{d_{\epsilon}(X_0)}} \right)\right] = \E_{X_0} \left[-\log\left ( \frac{C(\epsilon)}{K\epsilon^{d_{\epsilon}(X_0)}} \right)\right] \nonumber \\ & \quad \quad \quad \quad \quad = \log(\epsilon)\Hbar(X) +  \log \frac{K}{C(\epsilon)}. 
\end{align}
\end{proof}

\begin{prop}\label{prop2_B_B}
Consider any arbitrary causal diagram (directed acyclic graph) consisting of an observable RV $X\in \R^d$, and a set of hidden variables which may or may not represent causes in $C = \{C_1,C_2,...,C_n\}$ where $C_i\in \R \ \forall i$. Furthermore, assume that if $(X_1,X_2,...,X_p)$ are the parent nodes of $Y$, then we have that $(X_1,X_2,...,X_p)\fr Y$. Given this, we have
\begin{equation}
    \Hbar(X) \leq \min k \  \mbox{ s.t. } \  (C_{a(1)},C_{a(2)},..C_{a(k)}) \fr X, 
\end{equation}
where $1 \leq a(i)\leq n$. Thus, the \textit{ID-Entropy} of $X$ is less than the minimum number of variables in $C$ such that $X$ is f-related to the same.  
\end{prop}

\begin{proof}
Given any causal diagram, let $\{C_{a(1)},C_{a(2)},..C_{a(k)}\}$ represent a subset of hidden cause variables $C$, such that $(C_{a(1)},C_{a(2)},..C_{a(k)}) \fr X$. Let $\tilde{C}=[C_{a(1)},C_{a(2)},..C_{a(k)}]\in \R^k$ be the concatenation of the $k$ causes into a $k$ dimensional vector. As $\tilde{C} \fr X$, it follows from P7 that $\Hbar(X) \leq \Hbar(\tilde{C})$. As $\Hbar(A)\leq dim(A)$ for any continuous RV $A$, we have that $\Hbar(X) \leq \Hbar(\tilde{C}) \leq k$. The result then follows by considering all possible subsets $\tilde{C}$ of $C$, which satisfy $(C_{a(1)},C_{a(2)},..C_{a(k)}) \fr X$. 
\end{proof}
\setcounter{definition}{7}
\begin{definition}\label{prop3_B_B}
(\textbf{Bottleneck for Auto-Encoders}) Consider the class of auto-encoders which can be represented sequentially as $X \fr T \fr \tilde{X}$, where $T$ is the encoding of $X$ as some latent feature representation. Given this, we first have that $  \Hbar(X)=\Hbar(T)\leq \Hbar^{\epsilon}(T) \leq dim(T)$.
Thus, we outline the condition for a \textit{relaxed} information bottleneck criterion as follows.
\begin{equation}
    \min \Hbar^{\epsilon}(T) \ \mbox{ s.t. } \  \tilde{X}=X.
\end{equation}

\end{definition}
Note that below, we provide further mathematical intuition for the proposed bottleneck principle for auto-encoders.
\begin{remark}
As $X=\tilde{X}$, we have that $T\fr X$. Applying P7, we then have that $\Hbar(X)\leq \Hbar(T)\leq \Hbar^{\epsilon}(T)$. Thus, among all possible latent spaces $T$ that satisfy $X \fr T \fr \tilde{X}$, we will have that  $\Hbar(X)\leq  \Hbar^{\epsilon}$. Thus minimizing $\Hbar^{\epsilon}$ subject to the constraint that the decoder has a precise fit on the training data, will yield the smallest upper bound on the ID-Entropy of the data. This leads to the subsequent bottleneck criterion.
\end{remark}

\begin{definition}\label{prop4_B}
(\textbf{Bottleneck for Classifiers}) Consider the class of feedforward supervised architectures which can be represented sequentially as $X \fr T \fr \tilde{Y}$, where $T$ represents any hidden layer within the architecture. Given this, we outline the conditions for a \textit{relaxed} information bottleneck criterion as follows.
\begin{equation}
    \min \Hbar^{\epsilon}(T) \ \mbox{ s.t. } \ \tilde{Y}=Y.
\end{equation}
\end{definition}
Below, we provide further mathematical intuition for the proposed bottleneck principle for classifiers.
\begin{remark}
As $Y=\tilde{Y}$, the quantity $\Hbar(X|T)$ is indicative of the level of compression in the features $T$, while being able to adequately generate the output labels. A natural bottleneck thus is to maximize $\Hbar(X|T)$, with the constraint that $T$ can still perfectly fit the given labels, which yields
\begin{equation}
    \max \Hbar(X|T) = \max \left(\Hbar(X) - \Hbar(T) \right) = \Hbar(X) - \left( \min \Hbar(T) \right). 
\end{equation}
The above follows from P7 and the definition of joint ID-Entropy, as $X \fr T$. Furthermore, we see that the maximization eventually yields a minimization of $\min \Hbar(T)$. As we have $\Hbar(T)\leq \Hbar^{\epsilon}(T)$, this leads to $\min \Hbar(T)\leq \min  \Hbar^{\epsilon}(T)$, yielding the final criteria. 
\end{remark}

\section{Proof of Theorem 1 and Corollary 1.1}

\begin{theorem}  \label{thm:unsup_B}
\textbf{(Auto-Encoders)} We consider a scenario with an auto-encoder network.  Let $S=\{X_1,X_2,..,X_m\}$, be the training data points. Also, let $\widehat{err}_{MSE}(S)$ represent the training mean-squared reconstruction error and $err_{MSE}(g(f_T(X)),X)=\mathbb{E}_{S}[\widehat{err}_{MSE}(S)]$ be the generalization error, where $f_T(X)$ represents the function at any layer $T$, and $g(.)$ represents the reconstruction function. Let $g_{true}(.)$ be the ground truth reconstruction function. We consider architectures which yield continuous $f_T$ and $g$. Given that $X\sim P$, let $d_{max}=\max_{X\sim P,i} \lVert X-X_i \rVert$. Then, as $m\xrightarrow[]{} \infty$, for Lipschitz continuous $g$, $g_{true}$ and $f_T$, and for some $C>0$ we have that
\begin{equation}
    err_{MSE}(g(f_T(X)),X) \leq \widehat{err}_{MSE}(S) + \frac{C d_{max}^2\Hbar(T)}{\Hbar(T)+2},
\end{equation}
where $C$ depends on Lipschitz constants of $g$, $g_{true}$ and $f_T$. 
\end{theorem}
\begin{proof}
For any $X_i$, let us define its local neighborhood set $N_i$ as 
\begin{equation}
N_i=\{X: ||X-X_i||<||X-X_j|| \  0\leq j\leq m, j\neq i\ \ \& P(X)>0\}.
\end{equation}

Let us also define a set $S_i$ around $X_i$ as,
\begin{equation}
S_i=\{X: ||X-X_i||<d_{max} \ \& P(X)>0\}.
\end{equation}

As $m\arrow \infty$, given the definition of $d_{max}$, we must then have that $N_i\subseteq S_i$. Next, let $g$ be $L_1$-Lipschitz and $g_{true}$ be $L_2$-Lipschitz. As $m\arrow \infty$, we can assume a locally uniform distribution over the set of points in the neighborhood set $N_i$, yielding
\begin{align}
&\mathbb{E}_{X\sim N_i}[||g(f_T(X)) -g_{true}(f_T(X))||^2 ]  \\
&\leq \mathbb{E}_{X\sim N_i}[||g(f_T(X_i)) -g_{true}(f_T(X_i))||^2 ] + \mathbb{E}_{X\sim N_i}[||g(f_T(X)) -g(f_T(X_i))||^2 ] \\
& \quad \quad \quad \quad + \mathbb{E}_{X\sim N_i}[||g_{true}(f_T(X)) -g_{true}(f_T(X_i))||^2 ] \\ 
&\leq ||g(f_T(X_i)) -g_{true}(f_T(X_i))||^2 + L_1\mathbb{E}_{X\sim N_i}[||f_T(X)-f_T(X_i)||^2]  \\ 
&  + L_2\mathbb{E}_{X\sim N_i}[||f_T(X)-f_T(X_i)||^2]\\
&=||g(f_T(X_i)) -g_{true}(f_T(X_i))||^2 + (L_1^2+L_2^2)\mathbb{E}_{X\sim N_i}[||f_T(X)-f_T(X_i)||^2]\\
&\leq ||g(f_T(X_i)) -g_{true}(f_T(X_i))||^2 + (L_1^2+L_2^2)\mathbb{E}_{X\sim S_i}[||f_T(X)-f_T(X_i)||^2]
\end{align}

Next, let the intrinsic dimensionality at $f_T(X_i)$ be $d_i$. For any $d_i$ dimensional hypersphere $S$ centered at the origin and of radius $r$, we can directly estimate $\mathbb{E}_{X\sim S}[||X||^2]= r^2\frac{d_i}{d_i+2}$. Furthermore, let $f_T$ be $L_3$ Lipschitz, we can similarly embed the local neighborhood around $f_T(X_i)$ within a hypersphere of radius $L_3d_{max}$, as $||f_T(X)-f_T(X_i)||\leq  L_3 ||X-X_i||=L_3d_{max}$. As $m\arrow \infty$, we can use the local uniformity of the distribution and use these results. With this, we then have 
\begin{align}
\mathbb{E}_{X\sim P}&[||g(f_T(X)) -g_{true}(f_T(X))||^2 ] = \sum_{i=1}^{m} \frac{\mathbb{E}_{X\sim N_i}[||g(f_T(X)) -g_{true}(f_T(X))||^2 ]}{m} \\
&\leq \sum_{i=1}^{m} \frac{||g(f_T(X_i)) -g_{true}(f_T(X_i))||^2 + (L_1^2+L_2^2)\mathbb{E}_{X\sim S_i}[||X-X_i||^2]}{m} \\ 
&=\sum_{i=1}^{m} \frac{||g(f_T(X_i)) -g_{true}(f_T(X_i))||^2}{m}
+ \sum_{i=1}^{m} \frac{(L_1^2+L_2^2)L_3^2d_{max}^2\frac{d_i}{d_i+2}}{m} \\ 
&\leq \widehat{err}_{MSE}(S) +  {(L_1^2+L_2^2)d_{max}^2\frac{\E_i[d_i]}{\E_i[d_i]+2}} = \widehat{err}_{MSE}(S) + \frac{C d_{max}^2\Hbar(T)}{\Hbar(T)+2},
\end{align}
where $C=(L_1^2+L_2^2)L_3^2$. 
\end{proof}
\begin{corollary} \label{corr:sup_B_B}
\textbf{(Classifiers)}  We are given training data points and labels $S=\{(X_1,y_1),(X_2,y_2),..,(X_m,y_m)\}$.  Also, let $\widehat{err}(S)$ represent the 0-1 loss on the training data and $err(g(f_T(X)),y)$ represent the generalization error, where $f_T(X)$ represents the function at any layer $T$, and $g(f_T(X))$ yields the network output.  Let $g_{true}(.)$ be the ground truth label generating function. We consider architectures which yield continuous $f_T$ and $g$. Then, as $m\xrightarrow[]{} \infty$, for probabilistically Lipschitz continuous $g$ and $g_{true}$ and Lipschitz continuous $f_T$, we have that
\begin{equation}
    err(g(f_T(X)),y) \leq \widehat{err}(S) + \frac{C_{p} d_{max}\Hbar(T)}{\Hbar(T)+1},
\end{equation}
where $C_{p}$ depends on the probabilistically Lipschitz constants of $g$ and $g_{true}$, and the Lipschitz constant of $f_T$. 
\end{corollary}
\begin{proof}
To prove this corollary, we proceed in the same manner as in the proof of Theorem 1, defining $N_i$, $S_i$ and all other relevant quantities in the first part of the proof. Let us denote the 0-1 loss via the term $L_{01}$, and the probabilistic Lipschitz constants for $g$ and $g_{true}$ as $Lp_1$ and $Lp_2$ respectively. Next, we compute the expected 0-1 loss on the neighborhood set $N_i$ as follows. 

\begin{align}
&\mathbb{E}_{X\sim N_i}[L_{01} \left( g(f_T(X)) ,g_{true}(f_T(X)) \right) ]  \\
&\leq \mathbb{E}_{X\sim N_i}[L_{01} \left( g(f_T(X_i)) ,g_{true}(f_T(X_i)) \right)  ] + \mathbb{E}_{X\sim N_i}[L_{01} \left( g(f_T(X)) ,g(f_T(X_i)) \right) ] \\
& \quad \quad \quad \quad + \mathbb{E}_{X\sim N_i}[L_{01} \left( g_{true}(f_T(X)) ,g_{true}(f_T(X_i)) \right) ] \\ 
&\leq L_{01} \left( g(f_T(X_i)) ,g_{true}(f_T(X_i)) \right) + Lp_1\mathbb{E}_{X\sim N_i}[||f_T(X)-f_T(X_i)||] \\ 
&+ Lp_2\mathbb{E}_{X\sim N_i}[||f_T(X)-f_T(X_i)||]\\
&=L_{01} \left( g(f_T(X_i)) ,g_{true}(f_T(X_i)) \right) + (Lp_1+Lp_2)\mathbb{E}_{X\sim N_i}[||f_T(X)-f_T(X_i)||]\\
&\leq L_{01} \left( g(f_T(X_i)) ,g_{true}(f_T(X_i)) \right) + (Lp_1+Lp_2)\mathbb{E}_{X\sim S_i}[||f_T(X)-f_T(X_i)||]
\end{align}
Next, let the intrinsic dimensionality at $f_T(X_i)$ be $d_i$. For any $d_i$ dimensional hypersphere $S$ centered at the origin and of radius $r$, one can directly estimate $\mathbb{E}_{X\sim S}[||X||]= r\frac{d_i}{d_i+1}$, where the distribution is uniform in $S$. Furthermore, let $f_T$ be $Lp_3$ Lipschitz, we can similarly embed the local neighborhood around $f_T(X_i)$ within a hypersphere of radius $L_3d_{max}$, as $||f_T(X)-f_T(X_i)||\leq  L_3 ||X-X_i||=L_3d_{max}$. As $m\arrow \infty$, we can use the local uniformity of the distribution and use these results. With this, we then have
\begin{align}
\mathbb{E}_{X\sim P}&[L_{01} \left( g(f_T(X)) ,g_{true}(f_T(X)) \right) ] = \sum_{i=1}^{m} \frac{\mathbb{E}_{X\sim N_i}[L_{01} \left( g(f_T(X)) ,g_{true}(f_T(X)) \right) ]}{m} \\
&\leq \sum_{i=1}^{m} \frac{L_{01} \left( g(f_T(X)) ,g_{true}(f_T(X)) \right)  + (Lp_1+Lp_2)\mathbb{E}_{X\sim S_i}[||f_T(X)-f_T(X_i)||]}{m} \\ 
&=\sum_{i=1}^{m} \frac{L_{01} \left( g(f_T(X_i)) ,g_{true}(f_T(X_i)) \right)}{m}
+ \sum_{i=1}^{m} \frac{(Lp_1+Lp_2)L_3d_{max}\frac{d_i}{d_i+1}}{m} \\ 
&\leq \widehat{err}(S) +  {(Lp_1+Lp_2)L_3d_{max}\frac{\E_i[d_i]}{\E_i[d_i]+1}} = \widehat{err}(S) + \frac{C_p d_{max}\Hbar(T)}{\Hbar(T)+1},
\end{align}
where $C_p=(Lp_1+Lp_2)L_3$. The result follows. 
\end{proof}

\end{document}